\documentclass[twocolumn]{autart}    % Enable this line and disable the 
\usepackage{graphicx}          % Include this line if your 
\usepackage{amsmath}
\usepackage{amsfonts}
\usepackage{amssymb}
\usepackage{hyperref}
\usepackage{algorithm}% <=================http://ctan.org/pkg/algorithms
\usepackage{algpseudocode}% <========== http://ctan.org/pkg/algorithmicx
\usepackage{booktabs}
\usepackage{mathtools}
\usepackage{tikz}
\usetikzlibrary{decorations.pathreplacing}
\usetikzlibrary{fadings}
\newcommand{\D}{{\mathop{}\!\mathrm{d}}} % für Integrad - d's.
\newcommand{\R}{\mathbb{R}}
\newcommand{\Q}{\mathbb{Q}}

\newcommand{\N}{\mathbb{N}}

\newcommand{\PP}{\mathbb{P}}

\newcommand{\E}{\mathbb{E}}

\newcommand{\T }{\mathcal{T}}
\newcommand{\A}{\mathcal{A}}
\newcommand{\X}{\mathcal{X}}
\newcommand{\rrightarrow}{\mathrel{\mathrlap{\rightarrow}\mkern1mu\rightarrow}}
\newcommand{\one}{ 1 \hspace{-3pt} \mathrm{l}} %

\newcommand{\aloc}{{a_{\operatorname{loc}}}}

\newcommand{\ab}{{\mathbf{a}}}
\newtheorem{exa}[thm]{Example}

\makeatletter
\newcommand{\removelatexerror}{\let\@latex@error\@gobble}
\makeatother

\begin{document}

\begin{frontmatter}
%\runtitle{Insert a suggested running title}  % Running title for regular 
                                              % papers but only if the title  
                                              % is over 5 words. Running title 
                                              % is not shown in output.

\title{Robust $Q$-learning Algorithm for Markov Decision Processes under Wasserstein Uncertainty\thanksref{footnoteinfo}} % Title, preferably not more 
                                                % than 10 words.

\thanks[footnoteinfo]{Corresponding author A. Neufeld}

\author[NTU]{Ariel Neufeld}\ead{ariel.neufeld@ntu.edu.sg},    % Add the 
\author[NUS]{Julian Sester}\ead{jul{\_}ses@nus.edu.sg}               % e-mail address 

\address[NTU]{NTU Singapore, Division of Mathematical Sciences, 21 Nanyang Link, Singapore 637371.}  % Please supply                                              
\address[NUS]{National University of Singapore, Department of Mathematics, 21 Lower Kent Ridge Road, 119077}             % full addresses

                             % keyword list or with the 
                                          % help of the Automatica 
                                          % keyword wizard

\begin{abstract}                          % Abstract of not more than 200 words.
We present a novel $Q$-learning algorithm tailored to solve distributionally robust Markov decision problems where the corresponding ambiguity set of transition probabilities for the underlying Markov decision process is a Wasserstein ball around a (possibly estimated) reference measure. 
We prove convergence of the presented algorithm and provide several examples also using real data to illustrate both the tractability of our algorithm as well as the benefits of considering distributional robustness when solving stochastic optimal control problems, in particular when the estimated distributions turn out to be misspecified in practice.
\end{abstract}

\end{frontmatter}

\section{Introduction}
The among practitioners popular and widely applied \emph{$Q$-learning} algorithm provides a tractable reinforcement learning methodology to solve Markov decision problems  (MDP). The $Q$-learning algorithm learns an optimal policy online via observing at each time the current state of the underlying process as well as the reward depending on the current (and possibly next) state when acting according to a (not necessarily) optimal policy and by assuming to act optimally after the next state. The observed rewards determine a function $Q$ depending on a state-action pair that describes the \emph{quality} of the chosen action when being in the observed state. After a sufficient amount of observations the function $Q$ then allows in each state to decide which actions possess the most \emph{quality}. In this way the $Q$-learning algorithm determines an optimal policy.

 The $Q$-learning algorithm was initially proposed in Watkins' PhD thesis (\cite{watkins1989learning}).  \cite{jaakkola1994convergence} and \cite{watkins1992q} then provided a rigorous mathematical proof of the convergence of the $Q$-learning algorithm to the optimal $Q$-value function using results from stochastic approximation theory (see e.g.\,\cite{dvoretzky1956stochastic} and \cite{robbins1951stochastic}). The design of the $Q$-learning algorithm as well as the proof of its convergence to the optimal $Q$-value both rely on the dynamic programming principle of the corresponding Markov decision problem, which allows to find an optimal policy for the involved infinite horizon stochastic optimal control problem by solving a one time-step optimization problem. We refer to \cite{al2007model}, \cite{angiuli2022reinforcement}, \cite{angiuli2021reinforcement},  \cite{cao2021deep}, \cite{charpentier2021reinforcement}, \cite{hambly2021recent}, \cite{huang2020deep}, \cite{jeong2019improving}, \cite{kolm2020modern}, \cite{naghibi2006application}, \cite{ning2021double}, and \cite{wang2021online} for various successful applications of the $Q$-learning algorithm.

Recently, there has been a huge focus in the literature starting from  the viewpoint that one might have an estimate of the correct transition probability of the underlying Markov decision process, for example through the empirical measure derived from past observed data, but one faces the risk of misspecifying the correct distribution and hence would like to consider a distributionally robust Markov decision process (compare \cite{bauerle2021distributionally_2}, \cite{bauerle2021q}, \cite{chen2019distributionally},
\cite{ElGhaouiNilim2005robust},  \cite{hakobyan2023distributionally}, \cite{li2023policy}, \cite{liu2022distributionally},
\cite{mannor2016robust}, \cite{neufeld2022markov}, \cite{panaganti2022sample}, \cite{si2020distributional}, \cite{si2020distributionally},  \cite{uugurlu2018robust}, \cite{wang2022policy}, \cite{wiesemann2013robust}, \cite{xu2010distributionally}, \cite{yang2017convex}, \cite{yang2021towards}, and \cite{zhou2021finite}), also called \emph{Markov decision process under model uncertainty}, where one maximizes over the worst-case scenario among all probability measures of an ambiguity set of transition probabilities. We also refer to, e.g, the following related distributionally robust stochastic control problems \cite{chen2019distributionally}, \cite{coulson2019regularized}, \cite{guo2018data}, \cite{uugurlu2018robust}, \cite{van2015distributionally}, \cite{wiesemann2014distributionally},  and \cite{yang2020wasserstein} beyond the MDP setting.
Indeed, as discussed in \cite{liu2022distributionally}, there is a common risk in practice that one cannot fully capture the probabilities of the real-world environment due to its complexity and hence the corresponding reinforcement learning algorithm will be trained based on  misspecified probabilities. In addition, there is the risk that the environment shifts between the training period and the testing period. This situation can often be observed in practice as the future evolution of random processes rarely behaves \emph{exactly} according to, for example, the observed historical evolution. One may think as a prime example of financial markets, where several financial crises revealed repeatedly that used models were strongly misspecified. We refer to \cite{liu2022distributionally} for further examples, e.g.\ in robotics, and a further general discussion on the need of considering distributionally robust Markov decision processes and corresponding reinforcement learning based algorithms.

While there has been a lot of contributions in the literature on distributionally robust Markov decision problems, only very recently, to the best of our knowledge, there has been a first $Q$-learning algorithm developed in \cite{liu2022distributionally}  to solve distributionally robust Markov decision problems. More precisely, in \cite{liu2022distributionally} the authors recently introduced a $Q$-learning algorithm tailored for distributionally robust Markov decision problems where the corresponding ambiguity set of transition probabilities consists of all probability measures which are $\varepsilon$-close to a reference measure with respect to the Kullback-Leibler (KL) divergence, and prove its convergence to the optimal robust Q-value function.

The goal of this paper is to provide a $Q$-learning algorithm which can solve distributionally robust Markov decision problems where the corresponding ambiguity set of transition probabilities for the underlying Markov decision process is a Wasserstein ball around a (possibly estimated) reference measure. We obtain theoretical guarantees of convergence of our $Q$-learning algorithm to the corresponding optimal robust $Q$-value function (see also \eqref{eq_definition_qstar}).  The design of our $Q$-learning algorithm combines the dynamic programming principle of the corresponding Markov decision process under model uncertainty (see, e.g., \cite{neufeld2022markov}) and a convex duality result for worst-case expectations with respect to a Wasserstein ball (see \cite{bartl2020computational}, \cite{blanchet2019quantifying}, \cite{gao2022distributionally}, \cite{mohajerin2018data}, and \cite{zhao2018data}). 

From an application point of view, considering the Wasserstein distance has the crucial advantage that a corresponding Wasserstein-ball consists of  probability measures which do not necessarily share the same support as the reference measure, compared to the KL-divergence, where by definition probability measures within a certain fixed distance to the reference measure all need to have a corresponding support included in the support of the reference measure. We highlight that from a structural point of view, our $Q$-learning algorithm is different than the one in \cite{liu2022distributionally}, which roughly speaking comes from the fact that the dual optimization problem with respect to the Wasserstein distance has a different structure than the corresponding one with respect to the KL-divergence.

We demonstrate in several examples also using real data that our \emph{robust} $Q$-learning algorithm determines \emph{robust} policies that outperform non-robust policies, determined by the classical $Q$-learning algorithm, given that the probabilities for the underlying Markov decision process turn out to be misspecified.

The remainder of the paper is as follows. In Section~\ref{sec_setting} we introduce the underlying setting of the corresponding Markov decision process under model uncertainty. In Section~\ref{sec_robust_q_learning} we present our new $Q$-learning algorithm and provide our main result: the convergence of this algorithm to the optimal robust $Q$-value function. Numerical examples demonstrating the applicability as well as the benefits of our $Q$-learning algorithm compared to the classical $Q$-learning algorithm are provided in Section~\ref{sec_examples}. All proofs and auxiliary results are provided in Appendix~\ref{sec_auxiliary_results} and~\ref{sec_proofs}, respectively

\section{Setting and Preliminaries}\label{sec_setting}
In this section we provide the setting and define necessary quantities to define our $Q$-learning algorithm for distributionally robust stochastic optimization problems under Wasserstein uncertainty.
\subsection{Setting}\label{subsec_setting}
Optimal control problems are defined on a state space containing all the states an underlying stochastic process can attain. We model this state space as a finite subset $\mathcal{X} \subset \R^d$ where $d \in \N$ refers to the dimension of the state space.
We consider the robust control problem over an infinite time horizon, hence the space of all attainable states in this horizon is given by the infinite Cartesian product
$
\Omega:=\mathcal{X}^{\N_0}=\mathcal{X}\times \mathcal{X} \times \cdots$,  with the corresponding $ \sigma$-algebra $\mathcal{F}: = 2^{\mathcal{X}} \otimes 2^\mathcal{X} \otimes \cdots$.
On $\Omega$ we consider a stochastic process that describes the states that are attained over time. To this end, we
let $\left(X_{t}\right)_{t\in \N_0}$ be the canonical process on $\Omega$, that is defined by $X_t\left(x_0,x_1,\dots,x_t,\dots\right):=x_t$ for each $\left(x_0,x_1,\dots,x_t,\dots\right) \in \Omega$,  $t \in \N_0$.

Given a realization $X_t$ of the underlying stochastic process at some time $t \in \N_0$, the outcome of the next state $X_{t+1}$ can be influenced through actions that are executed in dependence of the current state $X_t$. At any time the  set of possible actions is given by a finite set $A \subseteq \R^m$,  where $m \in \N$ is the dimension of the action space (also referred to as control space). The set of admissible policies $\mathcal{A}$ over the entire time horizon contains all sequences of actions that depend at any time only on the current observation of the state process $(X_t)_{t\in \N_0}$ formalized by
\begin{align*}
\mathcal{A}:&=\bigg\{\ab=(a_t)_{t \in \N_0}~\bigg|~(a_t)_{t \in \N_0}: \Omega \rightarrow A;\\
&\hspace{1.8cm} a_t \text{ is } \sigma(X_{t})\text{-measurable} \text{ for all } t \in \N_0 \bigg\}\\
&=\bigg\{\left(a_t(X_t)\right)_{t\in \N_0}~\bigg|~ a_t:\mathcal{X} \rightarrow A \text{ Borel measurable} \\
&\hspace{5.3cm}\text{  for all }  t \in \N_0 \bigg\}.
\end{align*}
The current state and the chosen action influence the outcome of the next state by influencing the probability distribution with which the subsequent state is realized. As we take into account model uncertainty we assume that the correct probability kernel is unknown and hence, for each given state $x$ and action $a$, we consider an ambiguity set of  probability distributions representing the set of possible probability laws for the next state.
We denote by $\mathcal{M}_1(\Omega)$ and $\mathcal{M}_1(\X)$ the set of probability measures on $(\Omega, \mathcal{F})$ and $(\X,2^\X)$ respectively, and we assume that an ambiguity set of probability measures is modelled by a set-valued map
\begin{equation}\label{eq_defn_pxa}
\mathcal{X} \times A \ni (x,a) \rrightarrow \mathcal{P}(x,a) \subseteq \mathcal{M}_1(\mathcal{X}).
\end{equation}
Hence, if at time $t\in \N_0$ the process $X_t$ attains the value $x\in \X$, and the agent decides to execute action $a \in A$, then  $\mathcal{P}(x,a)$ describes the set of possible probability distributions with which the next state $X_{t+1}$ is realized. If $\mathcal{P}(x,a)$ is single-valued, then the state-action pair $(x,a)$ determines unambiguously the transition probability, and the setting coincides with the usual setting used for classical (i.e., non-robust) Markov decision processes, compare e.g. \cite{bauerle2011markov}.

The ambiguity set of admissible probability distributions on $\Omega$ depends therefore on the initial state $x\in X$ and the chosen policy $\ab \in \mathcal{A}$. We define for every initial state $x \in \mathcal{X}$ and every policy $\ab \in \mathcal{A}$ the set of admissible underlying probability distributions of $(X_t)_{t\in \N_0}$ by
\begin{align*}
\mathfrak{P}_{x,\ab}:=\bigg\{\delta_x \otimes \PP_0\otimes \PP_1 \otimes \cdots~\bigg|~&\text{ for all } t \in \N_0: \\
&\hspace{-3cm}\PP_t:\mathcal{X} \rightarrow \mathcal{M}_1(\mathcal{X}) \text{ Borel-measurable, } \\ 
&\hspace{-3cm}\text{ and }\PP_t(x_t) \in  \mathcal{P}(x_t,a_t(x_t)) \text{ for all } x_t\in \mathcal{X} \bigg\},
\end{align*}
where the notation $\PP=\delta_x \otimes\PP_0\otimes \PP_1 \otimes\cdots \in \mathfrak{P}_{x,\ab}$ abbreviates
\begin{align*}
\PP(B):=\sum_{x_0 \in \mathcal{X}}&\cdot \sum_{x_t \in \mathcal{X}} \cdots \one_{B}\left((x_t)_{t\in \N_0}\right) \cdots \PP_{t-1}(x_{t-1};\{ x_t\})\\
& \cdots \PP_0(x_0;\{x_1\}) \delta_x(\{x_0\}), \qquad B \in \mathcal{F}.
\end{align*}
\begin{rem}
In the literature of robust Markov decision processes one  refers to $\mathfrak{P}_{x,\ab}$ as being $(s,a)$-rectangular, see, e.g., \cite{iyengar2005robust}, \cite{shapiro2016rectangular},  \cite{wiesemann2013robust}. This is a common assumption which turns out to be crucial to obtain a dynamic programming principle (see, e.g.,  \cite[Theorem 2.7]{neufeld2022markov} and \cite{ruszczynski2006conditional}) and therefore to enable efficient and tractable computations. Indeed, if one weakens this assumption the problem becomes computationally more expensive (see, e.g, \cite[Section 2]{behzadian2021fast}), or can be provably intractable (compare \cite{li2023policy}) and therefore cannot be solved by dynamic programming methods. Several approaches to solve robust MDPs w.r.t.\,non-rectangular ambiguity sets using  methods other than dynamic programming however have recently been proposed, and are described in \cite{goyal2023robust}, \cite{li2023policy}, and \cite{tirinzoni2018policy}. 
\end{rem}
To determine \emph{optimal} policies we reward actions in dependence of the current state-action pair and the subsequent realized state. To this end, let $r:\mathcal{X}\times A \times \mathcal{X} \rightarrow \R$ be some \emph{reward function}, and let $\alpha \in \R$ be a \emph{discount factor} fulfilling 
\begin{equation} \label{eq_condition_alpha}
0 < \alpha < 1.
\end{equation}
Then, our \emph{robust} optimization problem consists, for every initial value $x\in \mathcal{X}$, in maximizing the expected value of $\sum_{t=0}^\infty \alpha^tr(X_{t},a_t,X_{t+1})$ under the worst case measure from $\mathfrak{P}_{x,\ab}$ over all possible policies $\ab \in \A$. More precisely, we aim for every $x\in \X$ to maximize 
$
\inf_{\PP \in \mathfrak{P}_{x,\ab}} \left(\E_{\PP}\bigg[\sum_{t=0}^\infty \alpha^tr(X_{t},a_t,X_{t+1})\bigg]\right)
$
among all policies $\ab \in \mathcal{A}$. The value function given by
\begin{equation}\label{eq_robust_problem_1}
\begin{aligned}
   \mathcal{X} \ni x \mapsto V(x):&=\sup_{\ab \in \mathcal{A}}\inf_{\PP \in \mathfrak{P}_{x,\ab}} \E_{\PP}\bigg[\sum_{t=0}^\infty \alpha^tr(X_{t},a_t,X_{t+1})\bigg]
\end{aligned}
\end{equation}
then describes the expectation of $\sum_{t=0}^\infty \alpha^tr(X_{t},a_t,X_{t+1})$ under the worst case measure from $\mathfrak{P}_{x,\ab}$ and under the optimal policy from $\ab \in \mathcal{A}$ in dependence of the initial value.
\subsection{Specification of the Ambiguity Sets}\label{sec_ambiguity_set}
To specify the ambiguity set $\mathcal{P}(x,a)$ for each $(x,a) \in \X \times A$, we first consider for each $(x,a) \in \X \times A$ a reference probability measure. In applications, this reference measure may be derived from observed data. Considering an ambiguity set related to this reference measure then allows to respect deviations from the historic behavior in the future and leads therefore to a more \emph{robust} optimal control problem that allows to take into account adverse scenarios, compare also \cite{neufeld2022markov}.
To that end, let
\begin{equation}\label{eq_defn_widehat_p}
\mathcal{X} \times A \ni (x,a) \mapsto \widehat{\PP}(x,a) \in \mathcal{M}_1(\mathcal{X}).
\end{equation}
be a probability kernel, where $\widehat{\PP}(x,a)$ acts as reference probability measure for each $(x,a) \in \X \times A$.
Then, for every $(x_0,\ab) \in \mathcal{X} \times \mathcal{A}$ we denote by
\begin{equation}\label{eq_defn_P_MDP}
\widehat{\PP}_{x_0,\ab} := \delta_{x_0} \otimes  \widehat{\PP}(\cdot,a_0(\cdot))\otimes \widehat{\PP}(\cdot,a_1(\cdot)) \otimes \cdots \in \mathcal{M}_1(\Omega)
\end{equation}
the corresponding probability measure on $\Omega$ that determines the distribution of $(X_t)_{t\in \N_0}$ in dependence of initial value $x_0 \in \X$ and the policy $\ab \in \mathcal{A}$, i.e., we have for any $B \in \mathcal{F}$ that
\begin{align*}
\widehat{\PP}_{x_0,\ab}(B):=&\sum_{x_0 \in \mathcal{X}}\cdots \sum_{x_t \in \mathcal{X}} \cdots \one_{B}\left((x_t)_{t\in \N_0}\right) \cdots \\
&\cdot\widehat{\PP}(x_{t-1},a_{t-1}(x_{t-1});\{ x_t\})\\
&\cdots \widehat{\PP}(x_0,a_0(x_0);\{x_1\}) \delta_x(\{x_0\}).
\end{align*}
We provide two specifications of ambiguity sets of probability measures $\mathcal{P}(x,a)$, $(x,a) \in \X \times A$, as defined in \eqref{eq_defn_pxa}. Both ambiguity sets rely on the assumption that for each given $(x,a) \in \X \times A$ the uncertainty with respect to the underlying probability distribution is modelled through a Wasserstein-ball around the reference probability  measure  $\widehat{\PP}(x,a)$ on $\mathcal{X}$.

To that end, for any  $q\in \N$, and any $\PP_1,\PP_2 \in \mathcal{M}_1(\mathcal{X})$, consider the $q$-Wasserstein-distance
\begin{align*}
W_q(\PP_1,\PP_2):&=\left(\inf_{\pi \in \Pi(\PP_1,\PP_2)}\int_{\mathcal{X} \times \mathcal{X}} \|x-y\|^q \D \pi(x,y)\right)^{1/q},
\end{align*}
where $\|\cdot \|$ denotes the Euclidean norm on $\R^d$ and where $\Pi(\PP_1,\PP_2) \subset \mathcal{M}_1(\mathcal{X} \times \mathcal{X})$ denotes the set of joint distributions of $\PP_1$ and $\PP_2$. Since we consider probability measures on a finite space we have a representation of the form
\[
\PP_i = \sum_{x \in \mathcal{X}} a_{i,x} \delta_x,\text{ with } \sum_{x \in \mathcal{X}} a_{i,x} = 1,~ a_{i,x} \geq 0 
\]
for all $x \in \X $ for $ i=1,2$,
where $\delta_x$ denotes the Dirac-measure at point $x \in \X$.
 Hence, the $q$-Wasserstein-distance can also be written as
\begin{align*}
W_q(\PP_1,\PP_2):&=\bigg(\min_{\pi_{x,y} \in \widetilde{\Pi}(\PP_1,\PP_2)}\sum_{x,y \in \mathcal{X}} \|x-y\|^q  \cdot \pi_{x,y}\bigg)^{1/q},
\end{align*}
where 
\begin{align*}
\widetilde{\Pi}(\PP_1,\PP_2):=\bigg\{\left(\pi_{x,y}\right)_{x,y \in \mathcal{X}}\subseteq [0,1]~&\bigg|~ \sum_{x'\in \X}\pi_{x',y} = a_{2,y}, \\
&\hspace{-2cm}\sum_{y' \in \X}\pi_{x,y'} = a_{1,x} \text{ for all } x,y \in \X \bigg\}.
\end{align*}
Relying on the above introduced Wasserstein-distance we define two ambiguity sets of probability measures.
\subsection*{\underline{Setting 1.)} The ambiguity set $\mathcal{P}_1^{(q,\varepsilon)}$}~\\
We consider for any fixed $\varepsilon>0$ and $q\in \N$  the  ambiguity set
\begin{equation}\label{eq_def_P1}
\begin{aligned}
\mathcal{X} \times A \ni (x,a) \rrightarrow \mathcal{P}_1^{(q,\varepsilon)}(x,a):=\bigg\{\PP\in \mathcal{M}_1(\mathcal{X})~&\text{s.t.}\\
&\hspace{-2cm}W_q(\PP,\widehat{\PP}(x,a)) \leq  \varepsilon \bigg\}
\end{aligned}
\end{equation}
being the $q$-Wasserstein ball with radius $\varepsilon$ around the reference measure $\widehat{\PP}(x,a)$, defined in \eqref{eq_defn_widehat_p}. For each $(x,a) \in \X \times A$ the ambiguity set $\mathcal{P}_1^{(q,\varepsilon)}(x,a)$ contains all probability measures that are close to $\widehat{\PP}(x,a)$ with respect to the $q$-Wasserstein distance. In particular, $\mathcal{P}_1^{(q,\varepsilon)}(x,a)$ contains also measures that are not necessarily dominated by the reference measure $\widehat{\PP}(x,a)$.
\subsection*{\underline{Setting 2.)} The ambiguity set $\mathcal{P}_2^{(q,\varepsilon)}$}~\\
We next define an ambiguity set that  can particularly be applied when autocorrelated time-series are considered. 
In this case we assume that the past $h \in \N \cap [2, \infty)$ values of a time series $(Y_t)_{t=-h+1,-h+2,\dots}$ may have an influence on the subsequent value of the state process. Then, at time $t\in \N_0$ the state vector is  given by
\begin{equation}\label{eq_x_equals_Y}
X_t = (Y_{t-h+1},\dots,Y_t)\in \X:= \mathcal{Y}^h\subset \R^{D \cdot h}, 
\end{equation}
$\text{with } \mathcal{Y}\subset \R^D \text{  finite, }$ where $D \in \N$ describes the dimension of each value $Y_t \in \mathcal{Y}\subset \R^D$.

 An example is given by financial time series of financial assets, where not only the current state, but also past realizations may influence the subsequent evolution of the assets and can therefore be modelled to be a part of the state vector, compare also the presentation in \cite[Section 4.3.]{neufeld2022markov}.
 
Note that at each time $t \in \N_0$ the part $(Y_{t-h+2},\dots,Y_{t}) \in \R^{D \cdot (h-1)}$ of the state vector $X_{t+1}$ that relates to past information can be derived once the current state $X_{t}=(Y_{t-h+1}, Y_{t-h+2},\dots,Y_{t})$ is known. Only the realization of $Y_{t+1}$ is subject to uncertainty. 
Conditionally on $X_{t}$ the distribution of $X_{t+1}$ should therefore be of the form $\delta_{\left(Y_{t-h+2},\dots,Y_{t}\right)}\otimes \widetilde{\PP}\in \mathcal{M}_1(\X)$ for some probability measure $\widetilde{\PP} \in \mathcal{M}_1(\mathcal{Y})$.

We write, given some $x =(x_1,\dots,x_h) \in \X$,
\begin{equation}\label{eq_pi_k}
\pi(x):=(x_2,\dots,x_h)\in \mathcal{Y}^{h-1}
\end{equation}
such that $x = \left(x_1,\pi(x)\right) \in\X$ and such that $\pi(X_t)=(Y_{t-h+2},\dots,Y_{t})$.
The vector $\pi(x)$ denotes the projection of $x$ onto the last $h-1$ components and represents the part of the state $x\in \X$ that is carried over to the subsequent state and is therefore not subject to any uncertainty.
To reflect the fact that the first $h-1$ components can be deterministicly derived once the previous state is known, we impose now the assumption that the reference kernel is of the form \begin{equation}\label{eq_defn_pi_P_hat_tilde}
\X \times A \ni (x,a)\mapsto \widehat{\PP}(x,a)=\delta_{\pi(x)} \otimes \widehat{\widetilde{\PP}}(x,a) \in \mathcal{M}_1(\X),
\end{equation}
where $\widehat{\widetilde{\PP}}$ is a probability kernel defined by
$
\mathcal{X} \times A \ni (x,a) \mapsto \widehat{\widetilde{\PP}}(x,a)\in \mathcal{M}_1(\mathcal{Y}).
$
This allows us to define for any fixed $\varepsilon>0$ and $q\in \N$ the ambiguity set\footnote{By abuse of notation $W_q$ here denotes the $q$-Wasserstein distance on $\mathcal{M}_1(\mathcal{Y})$.}
\begin{equation}\label{eq_def_P2}
\begin{aligned}
\X \times A \ni (x,a) \rrightarrow \mathcal{P}_2^{(q,\varepsilon)}(x,a):=\Bigg\{&\PP \in \mathcal{M}_1(\mathcal{X})~\text{ s.t. }~\\
&\hspace{-5.5cm}\PP= \delta_{\pi (x)} \otimes \widetilde{\PP}\text{ for } \widetilde{\PP} \in \mathcal{M}_1(\mathcal{Y}) \text{ with} ~W_q(\widetilde{\PP},\widehat{\widetilde{\PP}}(x,a))\leq \varepsilon \Bigg\},
\end{aligned}
\end{equation}
i.e., for each $(x,a) \in \X \times A$ we consider all measures of the form $\delta_{\pi(x)} \otimes \widetilde{\PP}$ for  $\widetilde{\PP}$ being close in the $q$-Wasserstein distance to $\widehat{\widetilde{\PP}}(x,a)$.

From now on, the ambiguity set of probability measures $\mathcal{P}(x,a)$, $(x,a) \in \X \times A$, either corresponds to $\mathcal{P}_1^{(q,\varepsilon)}(x,a)$, $(x,a) \in \X \times A$, defined by \eqref{eq_def_P1}, or to $\mathcal{P}_2^{(q,\varepsilon)}(x,a)$, $(x,a) \in \X \times A$, defined by \eqref{eq_def_P2}.
\begin{rem}\label{rem_historical}
In various applications such as for example portfolio optimization in finance (\cite[Section 4]{neufeld2022markov}), an agent would like to choose at each time $t$ an action $a_t$ not only based on the current observation of the state process but also based on some historical observations. To be able to cover such a scenario also in the context of Markov Decision Problems, it is a well-known procedure to extend the state space to be able to also include historical observations into the current state. The ambiguity set $\mathcal{P}_2^{(q,\varepsilon)}$ can therefore be seen as the natural extension of $\mathcal{P}_1^{(q,\varepsilon)}$ tailored exactly for that scenario described above. We highlight that in that case, given an agent observes $X_t= (Y_{t-h+1},\dots,Y_t)$ at time $t$, the only uncertainty on $X_{t+1}= (Y_{t-h+2},\dots,Y_t, Y_{t+1})$ lies in the last component $Y_{t+1}$, and not in the whole vector $X_{t+1}$, as the other components are observed through $X_t$.  This explains the structure of the corresponding measures in $\mathcal{P}_2^{(q,\varepsilon)}$ involving Dirac measures.
\end{rem}
\begin{rem}\label{rem_wasserstein}
Using the Wasserstein distance for capturing distributional uncertainty differs significantly from employing the Kullback-Leibler distance, which was used, e.g., in \cite{liu2022distributionally}. By using an ambiguity set defined via the Wasserstein distance, one can consider all probability distributions that are in proximity to a reference measure, even if they are not necessarily absolutely continuous with respect to it. This becomes important when the reference measure is estimated from historical data and contains point masses at the observed values, but one does not want to restrict future values to those observed in the past. In contrast, if one is confident about the support of the underlying transition kernel, it can be advantageous to use an ambiguity set defined using a distance such as the Kullback-Leibler distance which only considers probability measures with the same support (or smaller) as the reference measure.
\end{rem}
\subsection{Definition of Operators}
We consider the following single time step optimization problem 
\begin{equation}\label{eq_bellman_v}
\mathcal{T}V(x):=\sup_{a \in A} \inf_{\PP \in \mathcal{P}(x,a)} \E_{\PP}\big[r(x,a,X_1)+\alpha  V(X_1)\big],~x\in \mathcal{X},
\end{equation}
where $ \mathcal{X}\ni x \mapsto V(x)$ is the value function defined in \eqref{eq_robust_problem_1},
and we define the optimal robust $Q$-value function by
\begin{equation}\label{eq_definition_qstar}
\begin{aligned}
\mathcal{X} \times A \ni  (x,a)&\mapsto Q^*(x,a):=\\
&\inf_{\PP\in \mathcal{P}(x,a)}\E_{\PP}\big[r(x,a,X_1)+\alpha  V(X_1)\big].
\end{aligned}
\end{equation}

Note that if \eqref{eq_condition_alpha} holds and $\mathcal{P}(x,a)$ is either $\mathcal{P}_1^{(q,\varepsilon)}(x,a)$ or $\mathcal{P}_2^{(q,\varepsilon)}(x,a)$ for all $(x,a) \in \X \times A$, then the values of $Q^*$ are finite, since for all $(x,a) \in \mathcal{X} \times A $ we have
\begin{equation}\label{eq_q*_finite}
\begin{aligned}
|Q^*(x,a)|\leq  &\inf_{\PP\in \mathcal{P}(x,a)}\E_{\PP}\big[|r(x,a,X_1)|+\alpha  |V(X_1)|\big] \\
&\leq \sup_{y \in \mathcal{X}}|r(x,a,y)|+\alpha  \sup_{y \in \mathcal{X}} |V(y)| < \infty,
\end{aligned}
\end{equation}
where the finiteness of $V$ follows from \cite[Theorem 2.7]{neufeld2022markov}.
Then we obtain as a consequence of the main result from \cite[Theorem 3.1]{neufeld2022markov} the following proposition showing that the infinite time horizon distributionally robust optimization problem defined in \eqref{eq_robust_problem_1} can be solved by the consideration of a suitable one time-step fixed point equation, which is the key result that allows to derive $Q$-learning type of algorithms. 

\begin{prop}\label{prop_V_equals_TV}
Assume that \eqref{eq_condition_alpha} holds and that the ambiguity set $\mathcal{P}(x,a)$ is either given by  $\mathcal{P}_1^{(q,\varepsilon)}(x,a)$ or $\mathcal{P}_2^{(q,\varepsilon)}(x,a)$ for all $(x,a) \in \X \times A$. Then for all $x\in \mathcal{X}$ we have 
$
\sup_{a \in A} Q^*(x,a)=\T V(x)=V(x),
$
where $\X \ni x \mapsto V(x)$ corresponds to the value function of the robust stochastic optimal control problem defined in \eqref{eq_robust_problem_1}.
\end{prop}

\section{The Robust $Q$-learning Algorithm} \label{sec_robust_q_learning}
In this section we present a novel robust $Q$-learning algorithm for the corresponding distributionally robust stochastic optimization problem~\eqref{eq_robust_problem_1} and prove its convergence.

A robust $Q$-learning algorithm intends to  approximate $Q^*(x,a)= \inf_{\PP \in \mathcal{P}(x,a)} \E_{\PP}\big[r(x,a,X_1)+\alpha  V(X_1)\big]$ which involves the minimization over an infinite amount of probability measures. Due to the particular choice of ambiguity sets \eqref{eq_def_P1} and \eqref{eq_def_P2} w.r.t.\ the Wasserstein-distance, we can transform this minimization problem into a tractable problem using a duality from, e.g.,  \cite{bartl2020computational}.

To this end, for a function $f:\mathcal{X} \rightarrow \R$ we define, as in \cite[Section 2]{bartl2020computational} or \cite[Section 5]{villani2008optimal} its $\lambda c$ - transform.

\begin{defn}[$\lambda c$-transform]\label{def_lambda_c}
Let $f:\mathcal{X} \rightarrow \R$, let $\lambda \geq 0$, and let $c: \mathcal{X}\times \mathcal{X} \rightarrow \R$. Then the $\lambda c$-transform of $f$ is defined by
$
\mathcal{X} \ni x \mapsto (f)^{\lambda c}(x):=\sup_{y \in \X} \left\{f(y)-\lambda \cdot c(x,y)\right\}.
$
\end{defn}

Indeed, the $\lambda c$-transform now allows to rephrase the optimization problem involved in the definition of $Q^*$ in more tractable terms involving only an expectation with respect to the reference kernel, compare also Proposition~\ref{prop_dual_lambda_c_transform}. We use this representation to define our \emph{robust} $Q$-learning algorithm which is summarized in Algorithm~\ref{algo_q_learning}.

\begin{algorithm}
\caption{Robust $Q$-learning}\label{algo_q_learning}
 \hspace*{\algorithmicindent} \textbf{Input} State space $\mathcal{X} \subset \mathbb{R}^d$; Control space $A \subset \mathbb{R}^m$; Reward function $r$; Discount factor $\alpha \in (0,1)$;  Kernel $\widehat{\mathbb{P}}$; Starting point $x_0$; Policy $\ab \in \mathcal{A}$; Cost function $c$ of the $\lambda c$-transform; Ambiguity parameter $\varepsilon > 0$; Parameter $q \in \mathbb{N}$ related to the Wasserstein-distance;  Sequence of learning rates $(\widetilde{\gamma}_t)_{t\in \mathbb{N}_0} \subseteq [0,1]$;
\begin{algorithmic}[1]
%\Procedure{RobustQLearning}{}
\State Initialize $Q_{0}(x,a)$ for all $(x,a) \in \mathcal{X} \times A$ to an arbitrary real value;
\State Initialize $\operatorname{visits}(x,a) \leftarrow 0$ for all $(x,a) \in \mathcal{X} \times A$;
\For{$t = 0,1,\cdots$}
\State Set for all $(x,a) \in \mathcal{X} \times A$:
\State 
\[
\hspace{-0.3cm}\operatorname{visits}(x,a) \gets 
\begin{cases}
\operatorname{visits}(x,a) + 1 &\text{if } (x,a)=(X_t,a_t(X_t)), \\
\operatorname{visits}(x,a) &\text{else};
\end{cases}
\]
\State Define the map:
\begin{align*}
&\gamma_t : \mathcal{X} \times A \times \mathcal{X} \rightarrow \mathbb{R},\\
&(x,a,x') \mapsto \gamma_t(x,a,x'):= \widetilde{\gamma}_{\text{visits}(x,a)} \mathbb{I}_{\{(x',a_t(x'))=(x,a)\}};
\end{align*}
\State For every $(x,a) \in \mathcal{X} \times A$ we set:
\begin{equation}\label{eq_defn_f_t}
\begin{aligned}
f_{t,(x,a)} : \mathcal{X} &\rightarrow \mathbb{R},\\
y &\mapsto r(x,a,y)+\alpha \max_{b \in A} Q_t(y,b);
\end{aligned}
\end{equation}
\State Choose $\lambda_t \in [0,\infty)$ which satisfies:
\begin{equation}
\label{eq_defn_lambda_t}
\begin{aligned}
&\mathbb{E}_{\widehat{{\PP}}(X_t,a_t(X_t))}\left[-(-f_t(X_t,a_t(X_t)))^{\lambda_t c}(X_{t+1})-\varepsilon^q \lambda_t\right]\\
&\hspace{-0.5cm}=\sup_{\lambda \geq 0} \mathbb{E}_{\widehat{{\PP}}(X_t,a_t(X_t))}\left[-(-f_t(X_t,a_t(X_t)))^{\lambda c}(X_{t+1})-\varepsilon^q \lambda\right];
\end{aligned}
\end{equation}
\State For all $(x,a) \in \mathcal{X} \times A$ we define the following update rule:
\begin{equation}
\label{eq_q_learning_c_transform}
\begin{aligned}
Q_{t+1}(x,a):&= Q_{t}(x,a)\\
 &+ \gamma_t(x,a,X_t) \cdot \bigg(-(-f_{t,(x,a)})^{\lambda_t c}(X_{t+1})\\
 &\hspace{3cm}-\varepsilon^q\lambda_t-Q_{t}(x,a)\bigg);
\end{aligned}
\end{equation}
\EndFor
\end{algorithmic}
 \hspace*{\algorithmicindent} \textbf{Output} A sequence $(Q_{t}(x,a))_{t\in \mathbb{N}_0,~x \in \mathcal{X},~a \in A}$

%\EndProcedure

\end{algorithm}

The update rule from \eqref{eq_q_learning_c_transform} in Algorithm~\ref{algo_q_learning} means that for all $(x,a)\in \mathcal{X} \times A$, $t\in \N_0$, we have
$Q_{t+1}(x,a) =Q_{t}(x,a) + \widetilde{\gamma}_{\operatorname{visits}(x,a)} \bigg(-(-f_{t,(x,a)})^{\lambda_tc}(X_{t+1})  -\varepsilon^q\lambda_t-Q_t(x,a)\bigg)$ if $(x,a)=(X_t,a_t(X_t))$ and $Q_{t+1}(x,a) =Q_{t}(x,a)$ else,
i.e., the update of $Q_{t+1}$ only takes that state-action pair into account which was realized by the process $(X_t)_{t\in \N}$. Further, note that Algorithm~\ref{algo_q_learning} assumes for each time $t\in \N_0$ the existence of some $\lambda_t \in [0, \infty)$ such that \eqref{eq_defn_lambda_t} holds. The following result ensures that this requirement is indeed fulfilled.
\begin{lem}\label{lem_choice_of_lambda}
Let $(x,a) \in \X \times A$, $t \in \N_0$, let $\widehat{\PP} \in \mathcal{M}_1(\mathcal{X})$ and recall $\mathcal{X} \ni y \mapsto f_{t,(x,a)}(y)$ defined in \eqref{eq_defn_f_t}. Further let $\X \times \X \ni (x,y) \mapsto c(x,y) \in [0, \infty]$ satisfy $\min_{y \in \X } c(x,y)= 0 $ for all $x \in \X$. Then, there exists some $\lambda^* \in [0, \infty)$ such that 
$
\E_{\widehat{\PP}}\left[-(-f_{t,(x,a)})^{\lambda^* c}(X_{1})-\varepsilon^q \lambda^*\right]=\sup_{\lambda \geq 0} \left(\E_{\widehat{\PP}}\left[-(-f_{t,(x,a)})^{\lambda c}(X_{1})-\varepsilon^q \lambda\right]\right).
$
\end{lem}
The following main result now shows that the function $(Q_{t})_{t\in \N_0}$ obtained as the output of  Algorithm~\ref{algo_q_learning} converges indeed against the optimal robust $Q$-value function $Q^*$ defined in \eqref{eq_definition_qstar}.

\begin{thm}\label{thm_q_learning_wasserstein}
Assume that \eqref{eq_condition_alpha} holds, and let $(x_0,\ab) \in \mathcal{X} \times \mathcal{A}$ such that
\begin{equation}\label{eq_gamma_infinity_conditions_2}
\begin{aligned}
&\sum_{t=1}^\infty \gamma_t(x,a,X_t)=\infty, ~\sum_{t=1}^\infty \gamma_t^2(x,a,X_t)<\infty\\
&\text{ for all } (x,a)\in \mathcal{X} \times A \qquad \widehat{\PP}_{x_0,\ab}-\text{almost surely}.
\end{aligned}
\end{equation}
\begin{itemize}
\item[(i)] Let the ambiguity set be given by $\mathcal{P}(x,a)= \mathcal{P}_1^{(q,\varepsilon)}(x,a)$ for all $(x,a) \in \X \times A$ for some $\varepsilon>0$ and $q\in \N$, and consider\footnote{The function $c_1$ is used to determine the $\lambda c$-transform in the algorithm, see \eqref{eq_defn_lambda_t} and \eqref{eq_q_learning_c_transform}.} $c_1: \mathcal{X}\times  \mathcal{X} \ni (x,y) \mapsto \|x-y\|^q$. Then, we have for all $(x,a) \in \mathcal{X} \times A$ that
\begin{equation*}
\lim_{t\rightarrow \infty } Q_{t}(x,a) = Q^*(x,a) \qquad \widehat{\PP}_{x_0,\ab}-\text{almost surely.\footnotemark}
\end{equation*}
\footnotetext{{For the definition of the probability measure $\widehat{\PP}_{x_0,\ab}$  we refer to \eqref{eq_defn_P_MDP}.}}
\item[(ii)]
Let $\X = T^h$ for some $h \in \N \cap [2,\infty)$ and $T \subset \R^D$ finite for some $D \in \N$, let the ambiguity set be given by $\mathcal{P}(x,a) = \mathcal{P}_2^{(q,\varepsilon)}(x,a)$ for all $(x,a) \in \X \times A$ for some $\varepsilon>0$ and $q\in \N$, and consider\footnote{The function $c_2$ is used to determine the $\lambda c$-transform in the algorithm, see \eqref{eq_defn_lambda_t} and \eqref{eq_q_learning_c_transform}.} 
$
c_2: \mathcal{X}\times  \mathcal{X} \ni (x,y) \mapsto \infty \cdot \one_{\{(x_1,\dots,x_{h-1})\neq (y_1,\dots,y_{h-1})\}}(x,y)+\|x_h-y_h\|^q$,
where $(x,y) =\left((x_1,\dots,x_h),(y_1,\dots,y_h) \right)$.
Then, we have for all $(x,a) \in \mathcal{X} \times A$ that
\[
\lim_{t\rightarrow \infty } Q_{t}(x,a) = Q^*(x,a) \qquad \widehat{\PP}_{x_0,\ab}-\text{almost surely.}
\]
\end{itemize}
\end{thm}

\begin{rem}
Note that condition \eqref{eq_gamma_infinity_conditions_2} can be ensured by considering a sequence of learning rates $(\widetilde{\gamma}_t)_{t\in \N_0} \subseteq [0,1]$ satisfying
\begin{equation}\label{eq_gamma_infinity_conditions}
\sum_{t=0}^\infty \widetilde{\gamma}_t=\infty, \qquad \sum_{t=0}^\infty \widetilde{\gamma}_t^2<\infty,
\end{equation}
and $(X_t)_{t\in \N_0}$ is a (positive) recurrent irreducible Markov decision process under $\widehat{\PP}_{x_0,\ab}$.
\end{rem}
\begin{rem}\label{rem_eps_greedy}
Note that in the non-robust case it has been empirically shown that an efficient choice for $\ab \in \mathcal{A}$ when applying $Q$-learning is given by the so called \emph{$\widetilde \varepsilon$-greedy policy}, see e.g. \cite[Chapter 9]{dixon2020machine}, \cite{mnih2015human}, or \cite{tokic2011value}. The \emph{$\widetilde \varepsilon$-greedy policy} $\ab:=(a_1,a_2,\dots)\in \mathcal{A}$ is, for $\widetilde \varepsilon>0$, $t\in \N_0$, defined by
\[
\mathcal{X} \ni x \mapsto a_t(x) :=\begin{cases}
\operatorname{argmax}_{b \in B} Q_t(x,b)&\text{prob. } 1- \widetilde \varepsilon, \\
a \sim \mathcal{U}(A) &\text{prob. } \widetilde \varepsilon,
\end{cases}
\]
where $a \sim \mathcal{U}(A)$ means that a random action $a$ is  chosen uniformly at random from the finite set $A$. A popular modification of the \emph{$\widetilde \varepsilon$-greedy policy} is to start with a relatively large $\widetilde \varepsilon$ and to decrease the value of $\widetilde \varepsilon$ over time, see, e.g., \cite{mnih2015human}.
\end{rem}

\begin{rem}\label{rem_action_from_q}
Note that from the optimal $Q$-value function one can infer 
$\X \ni x \mapsto \aloc^*(x):= \operatorname{argmax}_{a \in A} Q^*(x,a)$ and 
$\ab^*:= (\aloc^*(X_0),\aloc^*(X_1),\dots) \in \mathcal{A}$ which solves the robust stochastic optimal control problem \eqref{eq_robust_problem_1}, compare Proposition~\ref{prop_V_equals_TV} and \cite[Theorem 2.7]{neufeld2022markov}. Analogously, by considering $\X \ni x \mapsto \operatorname{argmax}_{a \in A}Q_t (x,a)$ for a sufficiently large $t \in \N$, we can derive an approximation of the optimal action.
\end{rem}

The following result based on \cite{neufeld2023bounding} shows that whenever an agent possesses a good enough guess about the true (but to her unknown) probability kernel $\PP^{\operatorname{true}}(x,a)$ so that  it is contained in the ambiguity set, one can 
%immediately 
bound the difference of the values of the robust and non-robust Markov decision problems. 
This is important since $\lim_{t\rightarrow \infty } Q_{t}(x,a) = Q^*(x,a) \ \widehat{\PP}_{x_0,\ab}-\text{a.s.\ }$ and $\sup_{a \in A} Q^*(x,a)=V(x)$, hence the following result also provides an upper bound on the sub-optimality of the performance of our robust $Q$-learning algorithm. We see that it can be controlled to be arbitrarily small when $\varepsilon\to 0$, as long as the agent possesses a good enough guess for 
% the true	but unknown probability kernel 
	$\PP^{\operatorname{true}}(x,a)$ as discussed above. 
Note that compared to \cite{neufeld2023bounding}, no regularity assumptions on the map $(x,a)\mapsto  \PP^{\operatorname{true}}(x,a)$ nor on the reward function are necessary due to the finiteness of both the state and action space.

\begin{prop}\label{prop_bounds}
Let $\varepsilon>0$, $q\in \N$, and let \begin{equation}\label{eq_non_robust_problem_1}
\begin{aligned}
   \X \ni x \mapsto V^{\operatorname{true}}(x):=\sup_{\ab \in \mathcal{A}} \left(\E_{{\PP}^{\operatorname{true}}_{x,\ab}}\bigg[\sum_{t=0}^\infty \alpha^tr(X_{t},a_t,X_{t+1})\bigg]\right),
\end{aligned}
\end{equation}
with $$\PP_{x,\ab}^{\operatorname{true}}:=\delta_{x} \otimes \PP^{\operatorname{true}} \otimes \PP^{\operatorname{true}} \otimes \PP^{\operatorname{true}} \otimes \PP^{\operatorname{true}}  \cdots \in \mathcal{M}_1(\Omega),$$
where $\X \times A \ni(x,a) \mapsto \PP^{\operatorname{true}}(x,a) \in  \mathcal{P}_i^{(q,\varepsilon)}(x,a)$, $i \in \{1,2\}$, for all $x\in\mathcal{X}, a \in A$.  Moreover, assume that
the discount factor satisfies (2.2) as well as $\alpha L_P <1$, where
\begin{equation}\label{Lp}
L_P:= \sup_{(x,a), (x’,a’)\in \mathcal{X}\times A: \atop  (x,a)\neq (x’,a’)} 
\frac{W_q \left(\PP^{\operatorname{true}}(x,a), \PP^{\operatorname{true}}(x’,a’)\right)}{\Vert x-x’ \Vert + \Vert a-a’ \Vert}.
\end{equation}

Then for any $x \in \X$ we have
\begin{equation}\label{eq_bound_main_thm}
0 \leq V^{\operatorname{true}}(x)-V(x) \leq 2 L_r \varepsilon \left(1+\alpha\right)\sum_{i=0}^\infty \alpha^i \sum_{j=0}^i(L_P)^j < \infty,
\end{equation}
where 
\begin{equation}\label{Lr}
L_r:= \sup_{\substack{(x_0,a,x_1)\in \mathcal{X}\times A\times \mathcal{X},\\ (x’_0,a’,x’_1)\in \mathcal{X}\times A\times \mathcal{X}: \\ (x_0,a,x_1)\neq  (x’_0,a’,x’_1) } }
\frac{ |r(x_0,a,x_1)- r(x’_0,a’,x’_1)| }{\Vert x_0-x_0’ \Vert + \Vert a-a’ \Vert+\Vert x_1-x_1’ \Vert }.
\end{equation}
\end{prop}

\section{Numerical Examples}\label{sec_examples}
In this section we provide three numerical examples that illustrate how the robust $Q$-learning Algorithm~\ref{algo_q_learning} can be applied to specific problems. The examples highlight that a distributionally robust approach can outperform non-robust approaches whenever the assumed underlying distribution of the non-robust Markov $Q$-learning approach turns out to be misspecified during the testing period.

The selection of examples in this section is intended to give a small impression on the broad range of different applications of $Q$-learning algorithms for stochastic optimization problems.
 We refer to  \cite{bauerle2011markov},  \cite{dixon2020machine}, and \cite{hambly2021recent} for an overview on several applications in finance and %as well as 
to \cite{sutton2018reinforcement} for a range of applications outside the world of finance.

\subsection{On the Implementation}
To apply  the numerical method from Algorithm~\ref{algo_q_learning}, we use for all of the following examples a
discount factor of $\alpha = 0.45$, an $\widetilde \varepsilon$-greedy policy with $\widetilde \varepsilon=0.1$ (compare Remark~\ref{rem_eps_greedy}), $q = 1$, and as a sequence of learning rates we use $\widetilde{\gamma}_t = \tfrac{1}{1+t}$ for $t\in \N_0$. Moreover, we train all implementations with $50\,000$ iterations. The parameter $\lambda_t$ from \eqref{eq_defn_lambda_t} is determined by maximizing the right-hand-side of \eqref{eq_defn_lambda_t} with a numerical solver relying on the Broyden–-Fletcher–-Goldfarb–-Shanno (BFGS) algorithm (\cite{broyden1970convergence}, \cite{fletcher1970new}, \cite{goldfarb1970family}, \cite{shanno1970conditioning}).
Further details of the implementation can be found under \href{https://github.com/juliansester/Wasserstein-Q-learning}{https://github.com/juliansester/Wasserstein-Q-learning}.
\subsection{Examples}

\begin{exa}[Coin Toss]\label{exa_coins}
We consider an agent playing the following game:
At each time $t\in \N_0$ the agent observes the result of $10$ coins that  either show heads (encoded by $1$) or tails (encoded by $0$). The state $X_t$ at time $t\in \N_0$ is then given by the sum of the heads observed in the $10$ coins, i.e., we have $\mathcal{X}:=\{0,\dots,10\}$. At each time $t$ the agent can bet whether the sum of the heads of the next throw strictly exceeds the previous sum (i.e. $X_{t+1}>X_t$), or whether it is strictly smaller (i.e. $X_{t+1}<X_t$).

If the agent is correct, she gets $1$ $\$$, if the agent is wrong she has to pay $1$ $\$$. The agent also has the possibility not to play. We model this by considering the reward function:
$
\X \times A \times \X \ni (x,a,x') \mapsto r(x,a,x'):=a\one_{\{x<x'\}}-a\one_{\{x>x'\}}-|a|\one_{\{x=x'\}},
$
where the possible actions are given by $ A:=\{-1,0,1\}$, where for example $a=1$ corresponds to betting $X_{t+1}>X_t$.
We then rely on Setting 1.) from Section~\ref{sec_ambiguity_set} and consider as a reference measure a binomial distribution with $n=10,p=0.5$, i.e., 
$\X \times A \ni (x,a) \mapsto \widehat{\PP}(x,a):= \operatorname{Bin}(10,0.5).$
We then define, according to Setting 1.) from Section~\ref{sec_ambiguity_set}, an ambiguity set, in dependence of $\varepsilon>0$, by 
\begin{equation}\label{eq_defn_ambiguity_set_epsilon}
	\begin{aligned}
		\mathcal{P}(x,a):=\!\bigg\{\PP \in \mathcal{M}_1(\X)~\text{ s.t. }\ 
		W_1\bigg(\PP,\operatorname{Bin}(10,0.5)\bigg)\leq \varepsilon \bigg\}
	\end{aligned}
\end{equation}
for every  $(x,a)\in \X \times A$.
%\begin{equation}\label{eq_defn_ambiguity_set_epsilon}
%\begin{aligned}
%\X \times A \ni (x,a) \mapsto \mathcal{P}(x,a):= \bigg\{&\PP \in \mathcal{M}_1(\X)~\text{s.t.} 
%\\
%&W_1\bigg(\PP,\operatorname{Bin}(10,0.5)\bigg)\leq \varepsilon \bigg\},
%\end{aligned}
%\end{equation}
Let $p \in [0,1]$. Then, we denote the cumulative distribution function of a $B(10,p)$-distributed random variable by $F_{10,p}$. Then we compute for the $1$-Wasserstein distance that
\begin{equation}\label{eq_wasserstein_binom}
\begin{aligned}
&W_1\bigg(\operatorname{Bin}(10,0.5),~\operatorname{Bin}(10,p)\bigg) \\&= \int_{\R}\left| F_{10,0.5}(x)-F_{10,p}(x) \right| dx \\
& = \int_0^\infty F_{10,\min\{p,0.5\}}(x)-F_{10,\max\{0.5,p\}}(x) dx \\
& = \int_{0}^\infty \left(1-F_{10,\max\{0.5,p\}}(x)\right) dx \\
&\hspace{1cm}-\int_{0}^\infty \left(1-F_{10,\min\{0.5,p\}}(x)\right) dx \\
&=10\cdot \max\{0.5,p\}-10\cdot\min\{0.5,p\}=10\cdot |0.5-p|,
\end{aligned}
\end{equation}
where the first equality of \eqref{eq_wasserstein_binom} follows e.g.\  from \cite[Equation (3.5)]{ruschendorf2007monge} and the second equality of \eqref{eq_wasserstein_binom} follows since $F_{10,\min\{0.5,p\}}(x) \geq F_{10,\max\{0.5,p\}}(x)$ for all $x\in \R$.
This means that all binomial distributions $\operatorname{Bin}(10,p)$ with $p \in \left[0.5-\tfrac{\varepsilon}{10},~0.5+\tfrac{\varepsilon}{10}\right]$ are contained in the ambiguity set\footnote{We highlight that of course the ambiguity set not only contains binomial distributions.}. The calculation from \eqref{eq_wasserstein_binom} gives a good indication how choosing a different value of $\varepsilon$ may influence the measures contained in the ambiguity set. We then train actions $\ab^{\operatorname{robust},\varepsilon} = (a_t^{\operatorname{robust},\varepsilon})_{t \in \N_0}\in \A$ according to the robust $Q$-learning approach proposed in Algorithm~\ref{algo_q_learning} for different values of $\varepsilon$, compare also Remark~\ref{rem_action_from_q}. Additionally we train an action $\ab^{\operatorname{non-robust}} =(a_t^{\operatorname{non-robust}})_{t \in \N_0}\in \A$ according to the \emph{classical} non-robust $Q$-learning approach, see, e.g., \cite{watkins1989learning}, where we assume that the underlying process $(X_t)_{t \in \N_0}$ develops according to the reference measure $\widehat{\PP}$. We obtain after applying Algorithm~\ref{algo_q_learning} the strategies depicted in Table~\ref{tbl_actions_coints}.
\begin{table}[h!]
\begin{center}{
\resizebox{\columnwidth}{!}{
\begin{tabular}{c|ccccccccccc} \toprule
 $X_t$    & 0 &1 &2 &3 &4 &5 &6 &7 &8 &9 &10   \\
\midrule
 $a_t^{\operatorname{non-robust}}(X_t)$ &1 	&1 	&1 	&1 	&1 	&0 	&-1 	&-1 	&-1 	&-1 &-1
 \\
$a_t^{\operatorname{robust},\varepsilon=0.5}(X_t)$ &1 	&1 	&1 	&0 	&0 	&0 	&0 &0	&-1 	&-1 	&-1 
 \\
 $a_t^{\operatorname{robust},\varepsilon=1}(X_t)$ &1 	&1 	&0 	&0 	&0 	&0 	&0 &0	&0 	&-1 	&-1 
 \\
 $a_t^{\operatorname{robust},\varepsilon=2}(X_t)$ &0 	&0 	&0 	&0 	&0 	&0 	&0 &0	&0 	&0	&0
 \\
 \bottomrule
\end{tabular}}}
\end{center}
\caption{The trained actions $a_t^{\operatorname{robust},\varepsilon=0.5}(X_t)$, $a_t^{\operatorname{robust},\varepsilon=1}(X_t)$, $a_t^{\operatorname{robust},\varepsilon=2}(X_t)$, and $a_t^{\operatorname{non-robust}}(X_t)$ in dependence of the realized state $X_t$ at time $t \in \N_0$.}\label{tbl_actions_coints}
\end{table}
In particular, we see that in comparison with the \emph{non-robust} action $\ab^{\operatorname{non-robust}}$, the \emph{robust} actions $\ab^{\operatorname{robust},\varepsilon}$ behave more carefully where a larger value of $\varepsilon$ corresponds to a more careful behavior, which can be clearly seen for $\varepsilon=2$, in which case the agent decides not to play for every realization of the state.

Then, we test the profit of the resultant actions $\ab^{\operatorname{robust},\varepsilon} $ and $\ab^{\operatorname{non-robust}}$ by playing $100\,000$ rounds of the game according to these actions. For simulating the  $100\,000$ rounds we assume an underlying binomial distribution $\PP_{\operatorname{true}}=\operatorname{Bin}(10,p_{\operatorname{true}})$ with a fix probability $p_{\operatorname{true}}$ for heads which we vary from $0.1$ to $0.9$.
We depict the cumulated profits of the considered actions in Table~\ref{tbl_toin_coss}.
\begin{table}[h!]
\begin{center}{\resizebox{\columnwidth}{!}{
\begin{tabular}{l|ccccccccc} \toprule
$p_{\operatorname{true}}$   	&0.1  	&0.2  	&0.3 	&0.4 	&0.5 	&0.6 	&0.7 		&0.8  	&0.9 \\ \midrule
Non-Robust 	&-31386 	&-18438 	&-1567 &\textbf{22892} 		&\textbf{35082}  	&\textbf{22956} 	&-656 	&-18374 	&-31091 \\ 
Robust, $\varepsilon = 0.5$ 	&-24728 &4554  	&\textbf{16491} 	&13323 	&9920  	&13170 		&\textbf{16825} 	&4451 	 	&-24427\\
Robust, $\varepsilon = 1$ 	&-8174 	&\textbf{15201} 	&11091	&4387  	&2050 	&4373 		&11139 	&\textbf{15276}  	&-7611 \\
Robust, $\varepsilon = 2$ &\textbf{0} 	&0 	&0 	&0 	&0 	&0	&0 	&0 	 	&\textbf{0} \\
\bottomrule
\end{tabular}}}
\end{center}
\caption{Overall Profit of the game described in Example~\ref{exa_coins} in dependence of different trained strategies and of the probability distribution $\PP_{\operatorname{true}}=\operatorname{Bin}(10,p_{\operatorname{true}})$ of the underlying process. The best performing strategy in each case is indicated with bold characters.}\label{tbl_toin_coss}
\end{table}
We observe that if the probability for heads $p_{\operatorname{true}}$ is similar as probability for heads in the reference measure ($p=0.5$), then the non-robust approach (w.r.t.\ $\widehat{\PP}(x,a):= \operatorname{Bin}(10,0.5)$) outperforms the robust approaches. If however the model with which the non-robust action was trained was clearly misspecified then  $\ab^{\operatorname{robust},\varepsilon}$ outperforms $\ab^{\operatorname{non-robust}}$. More precisely, the larger the degree of misspecification the more favorable it becomes to choose a larger $\varepsilon$.
This can be well explained by the choice of the ambiguity set that covers, according to \eqref{eq_wasserstein_binom}, the more measures under which we test, the larger we choose $\varepsilon$. 

This simple example showcases that if in practice one is uncertain about the correct law according to which the state process evolves and one faces the risk of misspecifying the probabilities, then it can be advantageous to rely on a distributionally robust approach, whereas the choice of the radius of the Wasserstein-ball is extremely important as it corresponds to the degree of misspecification one wants to be robust against.
\end{exa}

\begin{exa}[Comparison with KL-Uncertainty]\label{exa_KL_vs_Wasserstein}
We reconsider an example of a supply-chain model provided in \cite[Section 4]{liu2022distributionally}. In this example we have for some $n \in \N$ the state space $\X = \{0,1,\dots,n\}$ representing the possible goods in the inventory and the action space $A = \X$ representing the possible goods we can order. The reward function is defined as the negative of the costs that are composed of holding costs and fixed ordering costs depending on parameters $h,p,k \in \R$ and on the demand which is, for the reference measure, uniformly distributed on $\X$ , see \cite[Section 4]{liu2022distributionally} for more details.

In the setting described in \cite[Section 4]{liu2022distributionally}, the optimal non-robust strategy (w.r.t.\ the reference measure) given current number of goods $x$ is $a_t^{\operatorname{non-robust}}(x)= (8-x) \cdot \one_{\{x \leq 2\}}$ while we compute for a Wasserstein-uncertainty parameter $\varepsilon = 1$ an optimal robust strategy $a_t^{\operatorname{Wasserstein}}(x)= (8-x) \cdot \one_{\{x \leq 1\}} +5 \one_{\{x \in \{2,3\} \}}$. The robust strategy computed in \cite[Section 4]{liu2022distributionally} that takes uncertainty w.r.t.\,Kullback--Leibler distance in account is given by $a_t^{\operatorname{KL}}(x)= (7-x) \cdot \one_{\{x \leq 4\}}$.

As in \cite[Figure 1]{liu2022distributionally}, we evaluate the strategies on  a distribution which does not coincide with the reference measure. To this end, we follow the example from \cite[Section 4]{liu2022distributionally} and consider a perturbed uniform distribution depending on parameters $m$ and $b$.
With parameter $b=1$ we compute after evaluation on $100~000$ iterations the  costs depicted in Figure~\ref{fig_KL_vs_Wasserstein}, in dependence of the parameter $m$. The figure shows that for this particular example the Wasserstein approach leads for all values  that are considered, except for $m \in \{5,6\}$, to smaller costs than the approach provided in \cite[Section 4]{liu2022distributionally}. Moreover, since  the true distribution does not coincide with the reference distribution, the robust strategies can outperform the non-robust ones (defined w.r.t.\ the reference distribution). % which are in that case misspecified.}

\begin{figure}
\centering
\includegraphics[scale=0.48]{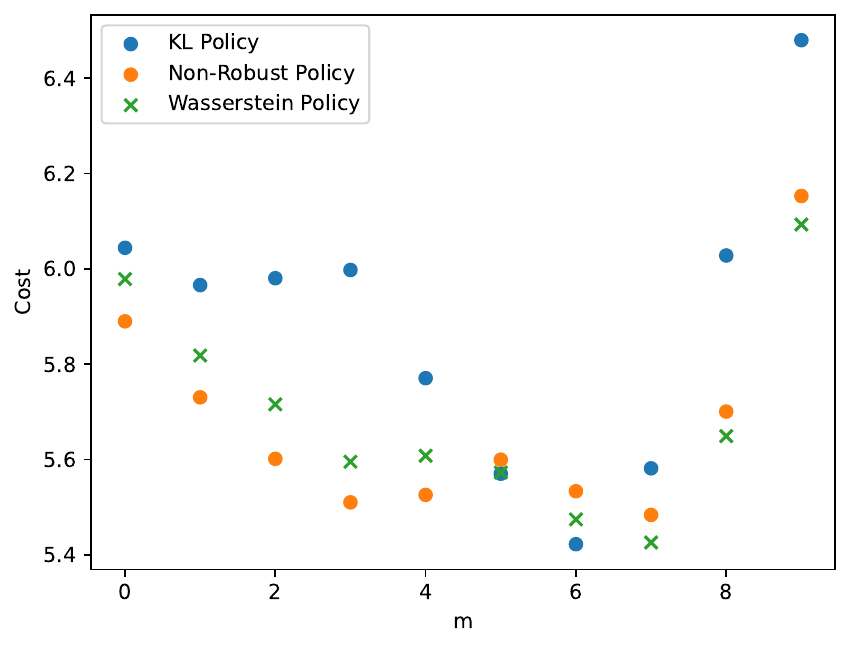}
\caption{ Total Costs for $b =1$ after $100 000$ iterations in the setting of Example~\ref{exa_KL_vs_Wasserstein}, compare also \cite[Figure 1]{liu2022distributionally}.} \label{fig_KL_vs_Wasserstein}
\end{figure}

\end{exa}

\begin{exa}[Stock Movement Prediction]
We study the problem of predicting the movement of stock prices. We aim to predict whether in the next time step the return of an underlying stock is strongly negative (encoded by $-2$), slightly negative (encoded by $-1$), slightly positive (encoded by $1$), or strongly positive (encoded by $2$). Hence the space of the numerically encoded returns is given by
$
T:= \{-2,-1,1,2\}.
$
We want to rely our prediction for the movement of the next return on the last $h=5$ values. Hence, we consider, in line with the setting outlined in \eqref{eq_x_equals_Y}
$
\X:=T^h=\{-2,-1,1,2\}^5.
$
The space of actions is modelled by 
$
A:= \{-2,-1,1,2\}=T
$
as the actions correspond to the future returns that are predicted. To construct a reference measure, we consider the historic evolution of the (numerically encoded) returns of the underlying stock. This time series is denoted by $(\mathcal{R}_{j})_{j=1,\dots,N} \subset T^N$ for some $N \in \N$, see also Figure~\ref{fig_dates} for an illustration.
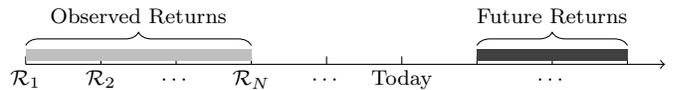
\begin{figure}[h!]
\begin{center}
\begin{tikzpicture}[%
    every node/.style={
        font=\scriptsize,
        text height=1ex,
        text depth=.25ex,
    },
]
% draw horizontal line   
\draw[->] (0,0) -- (8.5,0);

% draw vertical lines
\foreach \x in {0,1,...,8}{
    \draw (\x cm,3pt) -- (\x cm,0pt);
}

% place axis labels
\node[anchor=north] at (0,0) {$\mathcal{R}_1$};
\node[anchor=north] at (1,0) {$\mathcal{R}_2$};
\node[anchor=north] at (2,0) {$\cdots$};
\node[anchor=north] at (3,0) {$\mathcal{R}_N$};
\node[anchor=north] at (4,0) {$\cdots$};
\node[anchor=north] at (5,0) {Today};
\node[anchor=north] at (6,0) {};
\node[anchor=north] at (7,0) {$\cdots$};
\node[anchor=north] at (8,0) {};

%% draw scale above
\fill[lightgray] (0,0.05) rectangle (3,0.2);
\fill[darkgray] (6,0.05) rectangle (8,0.2);

% draw curly braces and add their labels
\draw[decorate,decoration={brace,amplitude=5pt}] (0,0.25) -- (3,0.25)
    node[anchor=south,midway,above=4pt] {Observed Returns};
\draw[decorate,decoration={brace,amplitude=5pt}] (6,0.25) -- (8,0.25)
    node[anchor=south,midway,above=4pt] {Future Returns};
\end{tikzpicture}
\end{center}
\caption{Illustration of the time relation between the time series of observed returns $(\mathcal{R}_j)_{j=1,\dots,N}$ and the future returns which we want to predict.}\label{fig_dates}
\end{figure}

We then define for some small\footnote{Note that $\gamma$ is only introduced to avoid a division by $0$.} $\gamma>0$ the set-valued map
$
\X \times A \ni (x,a)\mapsto \widehat{\widetilde{\PP}}(x,a):=\sum_{i\in T} p_i(x) \cdot  \delta_{\{i\}} \in \mathcal{M}_1(T)
$
where for $x\in \X, i \in T$ we define
\begin{equation}\label{eq_exa_stock_pi}
p_i(x):= \frac{\widetilde{p}_i(x)+\tfrac{\gamma}{4}}{\gamma+\sum_{j\in T} \widetilde{p}_j(x)} \in [0,1],
\end{equation}
as well as\footnote{Note that $\pi$ is defined in \eqref{eq_pi_k}.}
\begin{equation}\label{eq_exa_stock_pitilde}
\widetilde{p}_i(x):= \sum_{j=1}^{N-h+1} \one_{\{(\pi(x),i)=(\mathcal{R}_j,\dots,\mathcal{R}_{j+h-1})\}}.
\end{equation}
This means the construction of $\widehat{\widetilde{\PP}}(x,a)$ relies, according to \eqref{eq_exa_stock_pitilde}, on the relative frequency of the sequence $(\pi(x),i)$ in the time series of past realized returns $(\mathcal{R}_{j})_{j=1,\dots,N}$. Equation \eqref{eq_exa_stock_pi} is then applied to convert the frequencies to probabilities.
Then, as a reference measure we consider, as in \eqref{eq_defn_pi_P_hat_tilde}, the set-valued map
\begin{equation}
\X \times A \ni (x,a)\mapsto \widehat{\PP}(x,a)=\delta_{\pi(x)} \otimes \widehat{\widetilde{\PP}}(x,a) \in \mathcal{M}_1(\X).
\end{equation}
Moreover, as a reward function we consider\footnote{Here $x':=(x'_1,\dots,x'_h) \in \X$, and hence, $x'_h$ denotes the last component of $x'$.}
$
\X \times A \times \X \ni (x,a,x') \mapsto r(x,a,x'):= \one_{\{x'_h=a\}},
$
i.e., we reward only correct predictions.
We apply the setting described above to real data. To this end, we consider as series of realized returns $(\mathcal{R}_j)_{j=1,\dots,N}$ the daily returns of the stock of \emph{Apple} in the time horizon from $1$ January $2010$ until $28$ September $2018$ and hence we take into account $N=2200$ daily returns.
To encode the observed market returns to values in $T$, we distinguish between small returns and large returns by saying that a daily return is strongly positive if it is larger than $0.01$. Analogously a daily return is strongly negative if smaller than $-0.01$. This leads to the  distribution of  returns as depicted in Table~\ref{tbl_return_training_set}.

\begin{table}[h!]
\begin{center}{\resizebox{\columnwidth}{!}{
\begin{tabular}{lc} \toprule
Type of Encoded Return~(Numerical Value)& Total Amount\\ \midrule
Strongly Negative Returns~(-2)   &404\\
Slightly Negative Returns~(-1) &637\\
Slightly Positive Returns~(1) &627\\
Strongly Positive Returns~(2)  &532\\
\bottomrule
\end{tabular}}}
\end{center}
\caption{The distribution of the numerically encoded daily returns of \emph{Apple} between  January $2010$ and September $2018$. The threshold to distinguish slightly positive (negative) returns from strongly positive returns is $0.01$ ($-0.01$).}\label{tbl_return_training_set}
\end{table}

We then train a non-robust action $\ab^{\operatorname{non-robust}} =(a_t^{\operatorname{non-robust}})_{t \in \N_0}\in \A$ according to the \emph{classical} non-robust $Q$-learning algorithm (\cite{watkins1992q}) as well as robust actions $\ab^{\operatorname{robust}} = (a_t^{\operatorname{robust}})_{t \in \N_0}\in \A$ according to Algorithm~\ref{algo_q_learning} that takes into account an ambiguity set defined in \eqref{eq_def_P2} with $\varepsilon=0.1$. Moreover, for comparison, we consider a \emph{trivial} action  $\ab^{\operatorname{trivial}} = (a_t^{\operatorname{trivial}})_{t \in \N_0}\in \A$ which always, independent of the state-action pair, predicts $-1$ since, according to Table~\ref{tbl_return_training_set},  $-1$ is the most frequent appearing value in the time series $(\mathcal{R}_j)_{j=1,\dots,N}$ .

We then evaluate the trained actions, in a small backtesting study, on realized daily returns of \emph{Apple} that occurred after the end of the training period. To this end, we consider an evaluation period from $1$ October 2018 until $26$ February $2019$ consisting of $100$ daily returns that are distributed according to Table~\ref{tbl_return_evaluation_set}.
\begin{table}[h!]
\begin{center}{\resizebox{\columnwidth}{!}{
\begin{tabular}{lc} \toprule
Type of Encoded Return~(Numerical Value) & Total Amount\\ \midrule
Strongly Negative Returns~(-2)   &29\\
Slightly Negative Returns~(-1) &21\\
Slightly Positive Returns~(1) &22\\
Strongly Positive Returns~(2)   &28\\
\bottomrule
\end{tabular}}}
\end{center}
\caption{The distribution of the numerically encoded daily returns of \emph{Apple} between  $1$ October $2018$ and $26$ February $2019$.}\label{tbl_return_evaluation_set}
\end{table}
We observe that in the evaluation period, in contrast to the training period, the large negative returns impose the largest class of appearing returns. Overall the distribution is significantly different from the distribution of the classes on the training data. We illustrate in Table~\ref{tbl_results_evaluation_stocks} the results of predictions of the actions $\ab^{\operatorname{trivial}},\ab^{\operatorname{non-robust}},  \ab^{\operatorname{robust}}$ evaluated in the evaluation period, and we observe that indeed the robust action $\ab^{\operatorname{robust}}$ outperforms the other two actions clearly in this period where the distribution of returns significantly differs from the distributions of the returns on which the actions were trained.
\begin{table}[h!]
\begin{center}{{
\begin{tabular}{lc} \toprule
Action & Share of Correct Predictions\\ \midrule
$\ab^{\operatorname{non-robust}}$       &23.40$\%$ \\
$ \ab^{\operatorname{robust}}$     &28.72$\%$ \\
$\ab^{\operatorname{trivial}}$          &21.27$\%$   \\
\bottomrule
\end{tabular}}}
\end{center}
\caption{The proportion of correct stock movement predictions in the evaluation period between  $1$ October $2018$ and $10$ January $2019$}\label{tbl_results_evaluation_stocks}
\end{table}
This showcases again that if there is the risk that the underlying distribution on which the actions were trained turns out to be misspecified, then it can be advantageous to use a robust approach.
\end{exa}

\begin{ack}                               % Place acknowledgements
Financial support by the MOE AcRF Tier 1 Grant \emph{RG74/21} and by the  Nanyang Assistant Professorship Grant (NAP Grant) \emph{Machine Learning based Algorithms in Finance and Insurance} is gratefully acknowledged.   % here.
\end{ack}

\bibliographystyle{plain}        % Include this if you use bibtex 
\bibliography{literature}
\appendix

\section{Auxiliary Results and Proofs}%\label{sec_proofs}
In Section~\ref{sec_auxiliary_results} we provide several useful results which then allow in Section~\ref{sec_proofs} to prove the main result from Section~\ref{sec_robust_q_learning}.

\subsection{Auxiliary Results}\label{sec_auxiliary_results}
To establish convergence of our $Q$-learning algorithm that was presented in Section~\ref{sec_robust_q_learning} we will make use of the following auxiliary result from stochastic approximation theory which was developed to prove the convergence of the classical $Q$-learning algorithm. We refer to \cite[Section~3]{jaakkola1994convergence} for a discussion of the advantage of the following result compared to classical results from stochastic approximation such as, e.g., \cite{dvoretzky1956stochastic}.
% which is adjusted to the setting presented in Section~\ref{sec_setting} and can be found in \cite[Lemma 1]{singh2000convergence}. Compare also \cite[Theorem]{dvoretzky1956stochastic}, \cite[ Theorem 1]{jaakkola1994convergence}, \cite[Lemma 12]{szepesvari1996generalized}, and \cite[Lemma 3]{van2007convergence} for similar results.
 Note that for any $f: \mathcal{X} \times A \rightarrow \R$, we write 
\begin{equation}\label{eq_defn_supremum_norm}
\|f \|_{\infty}:= \sup_{x\in \mathcal{X}} \sup_{a\in \A} |f(x,a)|.
\end{equation}
\begin{lem}[\cite{jaakkola1994convergence}, Theorem 1] \label{lem_convergence}
Let $\PP_0 \in \mathcal{M}_1(\Omega)$ be a probability measure on $(\Omega, \mathcal{F})$, and consider a family of stochastic processes $(\gamma_t(x,a),F_t(x,a),\Delta_t(x,a))_{t\in \N_0}$, $(x,a)\in \X \times A$, satisfying for all $t \in \N_0$
\[
\Delta_{t+1}(x,a)=\left(1-\gamma_t(x,a)\right){\Delta}_t(x,a)+\gamma_t(x,a)F_t(x,a)
\]
$\PP_0\text{-almost surely for all } (x,a) \in \X \times A$.
Let $(\mathcal{G}_t)_{t\in \N_0}\subseteq \mathcal{F}$ be a sequence of increasing $\sigma$-algebras such that for all $(x,a) \in \X \times A$ the random variables $\Delta_0(x,a)$ and $\gamma_0(x,a)$ are $\mathcal{G}_0$-measurable and such that
$\Delta_t(x,a)$, $\gamma_t(x,a)$, and $F_{t-1}(x,a)$ are $\mathcal{G}_t$-measurable for all $t\in \N$.
Further assume that the following conditions hold.
\begin{itemize}
\item[(i)] $0 \leq \gamma_t(x,a ) \leq 1$, $\sum_{t=0}^\infty \gamma_t(x,a )= \infty$, $\sum_{t=0}^\infty \gamma_t^2(x,a) <\infty$ $\PP_0$-almost surely for all $(x,a) \in \X \times A$, $ t\in \N_0$.
\item[(ii)] There exists $\delta \in (0,1)$ such that $\|\E_{\PP_0}\left[F_t (\cdot ,\ \cdot )~\middle|~\mathcal{G}_t\right]\|_\infty\leq \delta \|\Delta_t\|_\infty$ $\PP_0$-almost surely for all $t\in \N_0$.
\item[(iii)] There exists $C>0$ such that $ \left\|\operatorname{Var}_{\PP_0}\left(F_t(\cdot ,\ \cdot  )~\middle|~\mathcal{G}_t\right) \right\|_{\infty}\leq C(1+\|\Delta_t\|_\infty)^2$ $\PP_0$-almost surely for all $ t\in \N_0$.
\end{itemize}
Then, $\lim_{t \rightarrow \infty} \Delta_t(x,a) =0 ~~ \PP_0$-almost surely for all $(x,a) \in \X \times A$.
\end{lem}
Next, as the following proposition shows, the $\lambda c$-transform  allows to compute worst case expectations with respect to probability measures contained in the Wasserstein-ball by computing its dual which solely depends on the center of the Wasserstein-ball.
\begin{prop}[\cite{bartl2020computational}, Theorem 2.4]\label{prop_dual_lambda_c_transform}
Let $f:\mathcal{X} \rightarrow \R$, let $\varepsilon>0$ and $q\in \N$, let $\mathcal{P}_1^{(q,\varepsilon)}, \mathcal{P}_2^{(q,\varepsilon)}$ be the ambiguity sets of probability measures defined in \eqref{eq_def_P1} and \eqref{eq_def_P2}, and let $c_1,c_2: \X \times \X \rightarrow [0,\infty]$ be defined as in Theorem~\ref{thm_q_learning_wasserstein}.
\begin{itemize}
\item[(i)]
Then, we have for every $(x,a) \in \mathcal{X} \times A$ that
$$
\inf_{\PP \in\mathcal{P}_1^{(q,\varepsilon)}(x,a)} \E_\PP[f]= \sup_{\lambda \geq 0}\left(\E_{\widehat{\PP}(x,a)}[-(-f)^{\lambda c_1}]-\varepsilon^q \lambda\right).
$$
\item[(ii)]
In addition, let $\X = T^h$ for some $h \in \N\cap [2,\infty)$, and $T \subset \R^D$ finite for some $D \in \N$. Moreover, assume that there exists some probability kernel $\X \times A \ni (x,a) \mapsto \widehat{\widetilde{\PP}}(x,a) \in \mathcal{M}_1(T)$ such  for all $(x,a) \in\X \times A$ we have  $\widehat{\PP}(x,a) = \delta_{\pi(x)}\otimes \widehat{\widetilde{\PP}}(x,a)$. Then, we have for every $(x,a) \in \mathcal{X} \times A$ that
$
\inf_{\PP \in \mathcal{P}_2^{(q,\varepsilon)}(x,a)} \E_\PP[f]= \sup_{\lambda \geq 0}\left( \E_{\delta_{\pi(x)}\otimes \widehat{\widetilde{\PP}}(x,a)}[-(-f)^{\lambda c_2}]-\varepsilon^q \lambda\right).
$
\end{itemize}

\end{prop}

\textbf{Proof of Proposition~\ref{prop_dual_lambda_c_transform}}\\
In case (i), the assertion follows by an application of the duality result from \cite[Theorem 2.4]{bartl2020computational} (with the specifications $d_{c_1}(\cdot,\cdot):= W_q(\cdot,\cdot)^q$, $\varphi(\cdot) := \infty\one_{(\varepsilon,
\infty]}(\cdot)$ in the notation of \cite[Theorem 2.4]{bartl2020computational}, see also  \cite[Example 2.5]{bartl2020computational}). More precisely, by \cite[Theorem 2.4]{bartl2020computational}, \cite[Example 2.5]{bartl2020computational} and by the definition of $\mathcal{P}_1^{(q,\varepsilon)}$ we have  for all $(x,a) \in \mathcal{X} \times A$ that
\begin{align*}
\inf_{\PP \in \mathcal{P}_1^{(q,\varepsilon)}(x,a)} \E_\PP[f]&= -\sup_{ \left\{\PP \in \mathcal{M}_1(\X)~\middle|~W_q\left(\PP, \widehat{\PP}(x,a)\right) \leq \varepsilon \right\}}\E_\PP[-f]\\
&= - \left(\inf_{\lambda \geq 0}\left\{\E_{\widehat{\PP}(x,a)}[(-f)^{\lambda c_1}]+\varepsilon^q \lambda \right\}\right)\\
&= \sup_{\lambda \geq 0} \left\{\E_{\widehat{\PP}(x,a)}[-(-f)^{\lambda c_1}]-\varepsilon^q \lambda \right\}.
\end{align*}
To show (ii), we observe that in the notation of \cite{bartl2020computational}, we have for $\PP,\PP'\in \mathcal{M}_1(\mathcal{X})$ that
\begin{align*}
d_{c_2}(\PP,\PP') &:= \inf_{\Q \in \Pi(\PP,\PP')} \int_{\X \times \X}  c_2(x,y)\D \Q(x,y)\\
&\hspace{-1cm}=\inf_{\Q \in \Pi(\PP,\PP')} \int_{\X \times \X}  \big(\infty \cdot \one_{\{(x_1,\dots,x_{h-1})\neq (y_1,\dots,y_{h-1})\}}\\
&+\|x_h-y_h\|^q \big)\D \Q \big((x_1,\dots,x_h),(y_1,\dots,y_h)\big).
\end{align*}
Hence, we have
$d_{c_2}(\PP,\delta_{\pi(x)} \otimes \widehat{\widetilde{\PP}}(x,a)) \leq \varepsilon^q$ if and only if $\PP =\delta_{\pi(x)} \otimes \widehat{\widetilde{\PP}}$ for some $\widehat{\widetilde{\PP}} \in \mathcal{M}_1(T)$ with $W_q(\widetilde{\PP},\widetilde{\PP}(x,a))\leq \varepsilon$. Moreover, we see that $c_2$ is indeed a cost function in the sense of \cite{bartl2020computational}.
This implies by \cite[Theorem 2.4]{bartl2020computational} and by the definition of $\mathcal{P}_2^{(q,\varepsilon)}$ that for all $(x,a) \in \X \times A$ we have
\begin{align*}
&\inf_{\PP \in \mathcal{P}_2^{(q,\varepsilon)}(x,a)} \E_\PP[f]\\&= -\sup_{\left\{\PP \in \mathcal{M}_1(\mathcal{X})~\middle|~d_{c_2} \left(\PP, \delta_{\pi(x)} \otimes \widehat{\widetilde{\PP}}(x,a)\right) \leq \varepsilon^q  \right\}}\E_\PP[-f]\\
&= - \left(\inf_{\lambda \geq 0}\left\{\E_{\delta_{\pi(x)} \otimes \widehat{\widetilde{\PP}}(x,a)}[(-f)^{\lambda c_2}]+\varepsilon^q \lambda \right\}\right)\\ &= \sup_{\lambda \geq 0} \left\{\E_{\delta_{\pi(x)} \otimes \widehat{\widetilde{\PP}}(x,a)}[-(-f)^{\lambda c_2}]-\varepsilon^q \lambda \right\}. \qed
\end{align*}

Next, consider the operator $H$ which is defined for any  $q: \mathcal{X}\times A \rightarrow \R$ by
\begin{equation}\label{eq_defn_HQ}
\begin{aligned}
\mathcal{X} \times A \ni (x,a) \mapsto &(Hq)(x,a):=\\
&\hspace{-1cm}\inf_{\PP \in \mathcal{P}(x,a)}\E_{\PP}\big[r(x,a,X_1)+\alpha  \max_{b\in A}q(X_1,b)\big].
\end{aligned}
\end{equation}
We derive for $H$ the following form of the Bellman-equation.
\begin{lem}\label{lem_bellman_q}
Assume that \eqref{eq_condition_alpha} holds and let the ambiguity set $\mathcal{P}$ be either $\mathcal{P}_1^{(q,\varepsilon)}$ or $\mathcal{P}_2^{(q,\varepsilon)}$, defined in \eqref{eq_def_P1} and \eqref{eq_def_P2}. Then the following equation holds true for the optimal $Q$-value function defined in \eqref{eq_definition_qstar}:
$
HQ^*(x,a)=Q^*(x,a)\text{ for all } (x,a)\in \mathcal{X} \times A.
$
\end{lem}
\textbf{Proof of Lemma~\ref{lem_bellman_q}}
This follows directly by definition of $Q^*$ and by Proposition~\ref{prop_V_equals_TV}. Indeed, let $(x,a)\in \mathcal{X} \times A$. Then, we have
\begin{align*}
(HQ^*)(x,a)&=\inf_{\PP \in \mathcal{P}(x,a)}\E_{\PP} \big[r(x,a,X_1)+\alpha  \sup_{b\in A}Q^*(X_1,b)\big]\\
&=\inf_{\PP \in \mathcal{P}(x,a)}\E_{\PP}\big[r(x,a,X_1)+\alpha  V(X_1)\big]\\
&=Q^*(x,a).\qed
\end{align*} 

Moreover, we observe that the operator $H$ is a contraction with respect to the supremum norm defined in \eqref{eq_defn_supremum_norm}.
\begin{lem}\label{lem_contraction}
For any maps $q_i: \mathcal{X}\times A \rightarrow \R$, $i=1,2$, we have
$
\left\|Hq_1-Hq_2\right\|_{\infty}\leq \alpha \left\| q_1-q_2 \right\|_{\infty}.
$
\end{lem}
\textbf{Proof of Lemma~\ref{lem_contraction}}
Consider any maps $q_i: \mathcal{X}\times A \rightarrow \R$, $i=1,2$. Then, we have for all $(x,a)\in \mathcal{X} \times A$ that
\begin{align*}
&\left|(Hq_1)(x,a)-(Hq_2)(x,a)\right|\\
&=\bigg|\inf_{\PP \in \mathcal{P}(x,a)}\E_{\PP}\big[r(x,a,X_1)+\alpha  \sup_{b\in A}q_1(X_1,b)\big]\\
&\hspace{1cm}-\inf_{\PP \in \mathcal{P}(x,a)}\E_{\PP}\big[r(x,a,X_1)+\alpha  \sup_{b\in A}q_2(X_1,b)\big] \bigg|\\
&\leq \sup_{\PP \in \mathcal{P}(x,a)}\bigg|\E_\PP\bigg[r(x,a,X_1)+\alpha \sup_{b\in A} q_2(X_1,b)\\
&\hspace{2.5cm}-r(x,a,X_1)-\alpha \sup_{b\in A} q_1(X_1,b)\bigg]\bigg|\\
&\leq \alpha  \sup_{\PP \in \mathcal{P}(x,a)}\E_\PP\left[\sup_{b\in A} \big|q_2(X_1,b)- q_1(X_1,b)\big|\right]\\
&\leq \alpha \sup_{(y,b) \in \mathcal{X} \times A} |q_2(y,b)-q_1(y,b)|= \alpha \|q_1-q_2\|_{\infty},
\end{align*}
which implies the assertion by taking the supremum with respect to the arguments of $Hq_1(\cdot,\cdot)-Hq_2(\cdot,\cdot)$. \qed
\subsection{Proofs}\label{sec_proofs}
In this section we provide the proofs of the results from Section~\ref{sec_setting} and Section~\ref{sec_robust_q_learning}.

\textbf{Proof of Proposition~\ref{prop_V_equals_TV}}\\
The first equality $\sup_{a \in A} Q^*(x,a)=\T V(x)$ follows by definition of $\T V$. For the second equality $\T V(x)=V(x)$ we 
want to check that \cite[Assumption 2.2]{neufeld2022markov} and \cite[Assumption 2.4]{neufeld2022markov} hold true to be able to
apply  \cite[Theorem 3.1]{neufeld2022markov}. \cite[Assumption 2.2]{neufeld2022markov} is fulfilled (for $p=0$ and $C_P=1$ in the notation of \cite[Assumption 2.2]{neufeld2022markov})  according to \cite[Proposition 3.1]{neufeld2022markov} in the case $\mathcal{P}=\mathcal{P}_1^{(q,\varepsilon)}$, and according to \cite[Proposition 3.3]{neufeld2022markov} in the case $\mathcal{P}=\mathcal{P}_2^{(q,\varepsilon)}$. To verify \cite[Assumption 2.4]{neufeld2022markov}~(i), note that $\mathcal{X}\times A \times \mathcal{X} \ni (x_0,a,x_1) \mapsto r(x_0,a,x_1)$ is continuous since $\mathcal{X}$ and $A$ are finite (endowed with the discrete topology). To show  \cite[Assumption 2.4]{neufeld2022markov}~(ii) note that for all $x_0,x_0',x_1 \in \mathcal{X}$ and $a,a' \in A$ we have
$
|r(x_0,a,x_1)-r(x_0',a',x_1)|\leq L \cdot \left(\|x_0-x_0'\|+\|a-a'\|\right).
$
with $L:=\left(\max_{y_0,y_0'\in \mathcal{X},~b,b'\in A \atop (y_0,b) \neq (y_0',b')} \frac{|r(y_0,b,x_1)-r(y_0',b',x_1)|}{\|y_0-y_0'\|+\|b-b'\|}\right)$.
Similarly, to show \cite[Assumption 2.4]{neufeld2022markov}~(iii), we observe that for all $x_0,x_1 \in \mathcal{X}$ and all $a \in A$ we have
\[
\left|r(x_0,a,x_1)\right| \leq \max_{y_0,y_1 \in \mathcal{X},~b\in A} |r(y_0,b,y_1)|,
\]
i.e., in the notation of \cite[Assumption 2.4]{neufeld2022markov} we have $C_r := \max\left\{1,~\max_{y_0,y_1 \in \mathcal{X},~b\in A} |r(y_0,b,y_1)|\right\}$. To verify \cite[Assumption 2.4]{neufeld2022markov}~(iv) we see that, since $p=0$, we can choose $C_P:=1$ in the notation of \cite[Assumption 2.2]{neufeld2022markov}~(ii) and hence with \eqref{eq_condition_alpha} we get
$
0< \alpha < 1 = \frac{1}{C_P}
$
as required. Hence, the result follows from \cite[Theorem 3.1]{neufeld2022markov}.
\qed

\textbf{Proof of Lemma~\ref{lem_choice_of_lambda}}
For any $\lambda\geq 0$ we have by definition of the $\lambda c$-transform
\begin{align*}
&\E_{\widehat{\PP}(x,a)}\left[-(-f_{t,(x,a)})^{\lambda c}(X_{1})-\varepsilon^q \lambda\right]  \\
= &\E_{\widehat{\PP}(x,a)}\left[-\max_{y \in \X} \left\{-f_{t,(x,a)}(y)-\lambda c(X_1,y) \right \}-\varepsilon^q \lambda\right].
\end{align*}
Therefore, since $\X$ is finite, the map 
$
[0,\infty) \ni \lambda \mapsto G(\lambda):= \E_{\widehat{\PP}(x,a)}\left[-(-f_{t,(x,a)})^{\lambda c}(X_{1})-\varepsilon^q \lambda\right] 
$
is continuous. Hence, the assertion of Lemma~\ref{lem_choice_of_lambda} follows once we have shown that $\lim_{ \lambda \rightarrow\infty} G(\lambda) = -\infty$. To that end, note that as, by assumption, $\min_{y \in \X } c(x,y)= 0 $ for all  $x \in \X$, we have that
\begin{align*}
&\limsup_{\lambda \rightarrow \infty} G(\lambda) \\ &= \limsup_{\lambda \rightarrow \infty}\E_{\widehat{\PP}(x,a)}\bigg[-\max_{y \in \X} \big\{-f_{t,(x,a)}(y)\\
&\hspace{3cm}-\lambda c(X_1,y) \big \}-\varepsilon^q \lambda\bigg]\\
&\leq \limsup_{\lambda \rightarrow \infty}\E_{\widehat{\PP}(x,a)}\bigg[\max_{z \in \X} f_{t,(x,a)}(z)\\
&\hspace{3cm}-\max_{y \in \X}\left\{-\lambda c(X_1,y) \right\}-\varepsilon^q \lambda\bigg] \\
& = \limsup_{\lambda \rightarrow \infty} \bigg(\max_{z \in \X} f_{t,(x,a)}(z)\\
&\hspace{2cm}+\lambda \E_{\widehat{\PP}(x,a)}\left[\min_{y \in \X} c(X_1,y) \right]-\varepsilon^q \lambda  \bigg)
\end{align*}
\begin{align*}
& = \max_{z \in \X} f_{t,(x,a)}(z)+\limsup_{\lambda \rightarrow \infty}  \left(-\varepsilon^q \lambda  \right) = -\infty.\qed
\end{align*}

\textbf{Proof of Theorem~\ref{thm_q_learning_wasserstein}}
%The proof follows a similar line of argumentation as the proofs of  convergence of non-robust $Q$-learning algorithm as they can be found for example in \cite[Theorem 1]{melo2001convergence} and \cite[Theorem 1]{singh2000convergence}.
Let $(x_0,\ab) \in \mathcal{X} \times \mathcal{A}$.\\

Assume that either $\mathcal{P} = \mathcal{P}_1^{(q,\varepsilon)}$ and $c = c_1$, or $\mathcal{P} = \mathcal{P}_2^{(q,\varepsilon)}$ and $c = c_2$. Then, we show for all $(x,a) \in \mathcal{X}\times A$
$
\lim _{ t \rightarrow \infty} Q_t(x,a) = Q^*(x,a)  \qquad \widehat{\PP}_{x_0,\ab}-\text{almost surely,}
$
which shows simultaneously both (i) and (ii).
To that end, let $(x,a) \in \mathcal{X}\times A$ be fixed. Then we rearrange the terms in \eqref{eq_q_learning_c_transform} and write 
\begin{equation}\label{eq_q_learning_finite_proof_1_c_transform}
\begin{aligned}
Q_{t+1}(x,a)= &(1-\gamma_t(x,a,X_t))Q_{t}(x,a) \\
&\hspace{-1cm}+ \gamma_t(x,a,X_t) \bigg(-(-f_{t,(x,a)})^{\lambda_t c}(X_{t+1})-\varepsilon^q \lambda_t \bigg),
\end{aligned}
\end{equation}
where $\lambda_t$ is as defined in \eqref{eq_defn_lambda_t}, see also Lemma~\ref{lem_choice_of_lambda} for its existence.
Note that by construction $Q_{t}(x,a) \in \R$ for all $(x,a) \in \X \times A$.
We define for every $t \in \N_0$ the map
$
\mathcal{X} \times A  \ni (x,a) \mapsto \Delta_t(x,a):= Q_{t}(x,a)-Q^*(x,a) \in \R.
$
Note that indeed, as for all $(x,a) \in \X \times A$ we have $Q_{t}(x,a)$ as well as $Q^*(x,a)$ is finite (compare \eqref{eq_q*_finite}), we directly conclude the finiteness of $\Delta_t(x,a)$ for all $(x,a) \in \X \times A$. 
Moreover, we obtain by  \eqref{eq_q_learning_finite_proof_1_c_transform} and by using the relation $\gamma_t(x,a,X_t)= \widetilde{\gamma}_t \one_{\{(X_t,a_t(X_t))=(x,a)\}}$ that
\begin{equation}\label{eq_q_learning_finite_proof_2_c_transform}
\begin{aligned}
\Delta_{t+1}(x,a)&= (1-\gamma_t(x,a,X_t))\Delta_t(x,a) \\
&\hspace{0.25cm}+ \gamma_t(x,a,X_t) \bigg(-(-f_{t,(x,a)})^{\lambda_t c}(X_{t+1})\\
&\hspace{3.25cm}-\varepsilon^q \lambda_t-Q^*(x,a)\bigg) \\
 &=(1-\gamma_t(x,a,X_t))\Delta_t(x,a) \\
 &\hspace{0.25cm}+ \gamma_t(x,a,X_t) \one_{\{(X_t,a_t(X_t))=(x,a)\}}\\ &\hspace{0.25cm}\cdot\bigg(-(-f_{t,(X_t,a_t(X_t))})^{\lambda_t c}(X_{t+1})\\
 &\hspace{2.5cm}-\varepsilon^q \lambda_t-Q^*(X_t,a_t(X_t))\bigg).
\end{aligned}
\end{equation}
Next, we define for every $t\in \N_0$ the random variable
$
F_t(x,a):=\one_{\{(X_t,a_t(X_t)=(x,a)\}}\bigg(-(-f_{t,(X_t,a_t(X_t))})^{\lambda_t c}(X_{t+1})-\varepsilon^q \lambda_t-Q^*(X_t,a_t(X_t))\bigg),
$
which by \eqref{eq_q*_finite} is finite for all $(x,a) \in \X \times A$. We consider the filtration $(\mathcal{G}_t)_{t\in \N_0}$ with  
$
\mathcal{G}_t:=\sigma \left( \left\{X_1,\dots,X_t\right\}\right),~~t\in \N,
$
and $\mathcal{G}_0=\{\emptyset, \Omega\}$ being the trivial sigma-algebra. Note that, in particular, $\Delta_t(x,a)$, $\gamma_t(x,a)$ and $F_{t-1}(x,a)$ are $\mathcal{G}_t$-measurable for all $t\in \N$. Moreover, we have by \eqref{eq_defn_P_MDP} and by Proposition~\ref{prop_dual_lambda_c_transform} that $\widehat{\PP}_{x_0,\ab}-\text{almost surely}$
\begin{align*}
\bigg|\E_{\widehat{\PP}_{x_0,\ab}}&[F_t(x,a)~|~\mathcal{G}_t]\bigg|\\
&=\bigg|\one_{\{(X_t,a_t(X_t)=(x,a)\}}\\
&\hspace{0.25cm}\cdot \E_{\widehat{\PP}(X_t,a_t(X_t))}\big[-(-f_{t,(X_t,a_t(X_t))})^{\lambda_t c}(X_{t+1})\\
&\hspace{3.5cm}-\varepsilon^q \lambda_t-Q^*(X_t,a_t(X_t))\big]\bigg|
\end{align*}
\begin{align*}
&= \one_{\{(X_t,a_t(X_t)=(x,a)\}}\\
&\hspace{0.25cm}\cdot \bigg|\sup_{\lambda \geq 0} \bigg(\E_{\widehat{\PP}(X_t,a_t(X_t))}\big[-(-f_{t,(X_t,a_t(X_t))})^{\lambda c}(X_{t+1})\\
&\hspace{3.5cm}-\varepsilon^q \lambda\big]\bigg)-Q^*(X_t,a_t(X_t))\bigg|\\
&=\one_{\{(X_t,a_t(X_t)=(x,a)\}}\\
&\hspace{0.25cm}\cdot \bigg|\inf_{\PP \in  \mathcal{P}(X_t,a_t(X_t))}\E_{\PP}\big[(f_{t,(X_t,a_t(X_t))})(X_{t+1})\big]\\
&\hspace{4cm}-Q^*(X_t,a_t(X_t))\bigg|.
\end{align*}
Thus, \eqref{eq_defn_f_t}, \eqref{eq_defn_HQ}, and
Lemma~\ref{lem_bellman_q} show that $\widehat{\PP}_{x_0,\ab}-\text{almost surely}$
\begin{align}
&\bigg|\E_{\widehat{\PP}_{x_0,\ab}}[F_t(x,a)~|~\mathcal{G}_t] \bigg| \notag \\
&= \one_{\{(X_t,a_t(X_t)=(x,a)\}} \notag  \\
&\hspace{0.5cm}\cdot \bigg|\inf_{\PP \in \mathcal{P}(X_t,a_t(X_t))}\E_{\PP}\big[(f_{t,(X_t,a_t(X_t))})(X_{t+1})\big] \notag  \\
&\hspace{5cm}-Q^*(X_t,a_t(X_t))\bigg| \notag
\end{align}
\begin{align}
&=\one_{\{(X_t,a_t(X_t)=(x,a)\}}\cdot \bigg|(HQ_t)(X_t,a_t(X_t)) \notag  \\
&\hspace{5cm}-Q^*(X_t,a_t(X_t))\bigg|\label{eq_HQ_Q}\\
&=\one_{\{(X_t,a_t(X_t)=(x,a)\}}\cdot \bigg|(HQ_t)(X_t,a_t(X_t)) \notag  \\
&\hspace{4.5cm}-(HQ^*)(X_t,a_t(X_t))\bigg|.  \notag
\end{align}
Hence it follows with Lemma~\ref{lem_contraction} that  $\widehat{\PP}_{x_0,\ab}-\text{almost surely}$
\begin{align*}
\left\|\E_{\widehat{\PP}_{x_0,\ab}}[F_t(\cdot, \cdot)~|~\mathcal{G}_t] \right\|_{\infty} &\leq \left\| (HQ_t)-(HQ^*)\right\|_{\infty}\\
&\leq \alpha \left\| Q_t-Q^*\right\|_{\infty} = \alpha \|\Delta_t\|_{\infty},
\end{align*}
where the norm $\| \cdot \|$ is defined in \eqref{eq_defn_supremum_norm}.
Next, recall that $C_r := \max \left \{1,~\max_{y_0,y_1 \in \mathcal{X},~b\in A} |r(y_0,b,y_1)|\right\}$. Note that by \eqref{eq_defn_f_t}, by the $\lambda c$-transform from Definition~\ref{def_lambda_c}, and since $\inf_{x,y\in \X} c(x,y)=0$, we have for all $t\in \N_0$ that
\begin{align*}
&\left(-f_{t,(X_t,a_t(X_t))}\right)^{\lambda_t c}(X_{t+1}) + \alpha \min_{y' \in \X} \max_{b' \in A} Q^*(y',b') \\
&= \left(-r(X_t,a_t(X_t),\cdot)-\alpha \max_{b \in A} Q_t(\cdot ,b) \right)^{\lambda_t c}(X_{t+1}) \\
&\hspace{2cm}+ \alpha \min_{y' \in \X} \max_{b' \in A} Q^*(y',b')
\end{align*}
\begin{align*}
&\leq \left(C_r-\alpha \max_{b \in A} Q_t(\cdot,b)\right)^{\lambda_t c}(X_{t+1})  \\
&\hspace{2cm}+ \alpha \min_{y' \in \X} \max_{b' \in A} Q^*(y',b') \\
&\leq\max_{z \in \X} \left(C_r-\alpha \max_{b \in A} Q_t(\cdot,b)\right)^{\lambda_t c}(z) \\
&\hspace{2cm}+\alpha \min_{y' \in \X} \max_{b' \in A} Q^*(y',b')
\end{align*}
The latter expression coincides with 
\begin{align*}
\max_{z,y\in \mathcal{X}} \bigg(C_r+ &\alpha \min_{y' \in \X} \max_{b' \in A} Q^*(y',b')\\
&-\alpha \max_{b \in A} Q_t(y,b)-\lambda_t c(z,y)\bigg),
\end{align*} which implies
\begin{equation}\label{eq_ft_ineq_1}
\begin{aligned}
&\left(-f_{t,(X_t,a_t(X_t))}\right)^{\lambda_t c}(X_{t+1}) + \alpha \min_{y' \in \X} \max_{b' \in A} Q^*(y',b')\\
&\leq \max_{z,y\in \mathcal{X}} \bigg(C_r+ \alpha  \max_{b' \in A} Q^*(y,b')  \\
&\hspace{2cm}-\alpha \max_{b \in A} Q_t(y,b)-\lambda_t c(z,y)\bigg)\\
&\leq \max_{z,y\in \mathcal{X}} \left(C_r+ \alpha  \max_{b \in A} \left\{ Q^*(y,b) - Q_t(y,b)\right\}-\lambda_t c(z,y)\right)\\
&\leq \max_{z,y\in \mathcal{X}}\left(C_r+ \alpha  \|\Delta_t\|_{\infty}-\lambda_t c(z,y)\right) \\
&= C_r+\alpha  \|\Delta_t\|_{\infty}=:M \in \R,
\end{aligned}
\end{equation}
and similarly, since $c(z,z)=0$ for all $z \in \X$,
\begin{equation*}
\begin{aligned}
&\left(-f_{t,(X_t,a_t(X_t))}\right)^{\lambda_t c}(X_{t+1}) +\alpha \min_{y' \in \X} \max_{b' \in A} Q^*(y',b')\\
&\geq \min_{z\in \mathcal{X}} \max_{y\in \mathcal{X}} \bigg(-C_r+\alpha \min_{y' \in \X} \max_{b' \in A} Q^*(y',b')  \\
&\hspace{3cm}-\alpha \max_{b \in A} Q_t(y,b)-\lambda_t c(z,y)\bigg)\\
&\geq \min_{z\in \mathcal{X}} \max_{y\in \mathcal{X}} \bigg(-C_r+\alpha \min_{y' \in \X} \min_{b \in A} \left( Q^*(y',b) -Q_t(y,b)\right)\\
&\hspace{5cm}-\lambda_t c(z,y)\bigg)\\
&\geq \min_{z\in \mathcal{X}} \max_{y\in \mathcal{X}} \bigg(-C_r-\alpha \max_{y' \in \X, b \in A} \left| Q_t(y,b)-Q^*(y',b) \right|\\
&\hspace{5cm}-\lambda_t c(z,y)\bigg)\\
&\geq \min_{z\in \mathcal{X}} \left(-C_r-\alpha \max_{y' \in \X, b \in A}\left| Q_t(z,b)-Q^*(y',b) \right|\right)\\
&\geq  -C_r-\alpha \max_{z,y' \in \X, b \in A} \bigg( \left|  Q_t(z,b)-Q^*(z,b)\right|\\
&\hspace{3cm}+\left| Q^*(z,b)-Q^*(y',b) \right|\bigg)
\end{aligned}
\end{equation*}
\begin{equation}\label{eq_ft_ineq_2}
\begin{aligned}
&\geq -C_r -\alpha \|\Delta_t\|_{\infty} \\
&\hspace{1cm}-\alpha \max_{z,y' \in \X, b \in A}\left|  Q^*(z,b)-Q^*(y',b)\right|=:m\in \R.
\end{aligned}
\end{equation}
We define $C:=( 4\alpha^2 +( 2C_r+\alpha \max_{z,y' \in \X, b \in A}| Q^*(z,b)-Q^*(y',b)|)^2 )< \infty$. Then, by using Popoviciu's inequality on variances\footnote{Popoviciu's inequality (see \cite{popoviciu1935equations} or \cite{sharma2010some}) states that for all random variables $Z$ on a probability space $(\Omega, \mathcal{F}, \PP)$ satisfying $m\leq Z \leq M$ for some $-\infty<m\leq M < \infty$ we have $\operatorname{Var}_{\PP}(Z) \leq \frac{1}{4}\left( M-m\right)^2$.} applied to the bounds $m$, $M$ computed in \eqref{eq_ft_ineq_1} and \eqref{eq_ft_ineq_2}, and by using the inequality  $(a+b)^2 \leq 2\left(a^2+b^2\right)$ which holds for all $a,b \in \R$, we see for every $(x,a) \in \X \times A$ that  $\widehat{\PP}_{x_0,\ab}-\text{almost surely}$
\begin{equation*}
\begin{aligned}
&\operatorname{Var}_{\widehat{\PP}_{x_0,\ab}}(F_t(x,a)~|~\mathcal{G}_t ) \\
&=\one_{\{(X_t,a_t(X_t)=(x,a)\}}\\
&\hspace{1cm}\cdot\operatorname{Var}_{\widehat{\PP}(X_t,a_t(X_t))}\left((-f_{t,(X_t,a_t(X_t))})^{\lambda_t c}(X_{t+1})\right)\\
&=\one_{\{(X_t,a_t(X_t)=(x,a)\}}\\
&\hspace{1cm}\cdot\operatorname{Var}_{\widehat{\PP}(X_t,a_t(X_t))}\bigg((-f_{t,(X_t,a_t(X_t))})^{\lambda_t c}(X_{t+1})\\
&\hspace{4cm}+\alpha \min_{y' \in \X} \max_{b' \in A} Q^*(y',b')\bigg)\\
&\leq \frac{1}{4}(M-m)^2 \\
&=\frac{1}{4} \bigg(2\alpha \|\Delta_t\|_{\infty} +2C_r \\
&\hspace{2cm}+\alpha \max_{z,y' \in \X, b \in A}\left| Q^*(z,b)-Q^*(y',b)\right|\bigg)^2\\
&\leq \frac{1}{2} \bigg( 4\alpha^2 \|\Delta_t\|_{\infty}^2 \\
&\hspace{0.5cm}+ \big( 2C_r+\alpha \max_{z,y' \in \X, b \in A}\left|Q^*(z,b)-Q^*(y',b)\right|\big)^2 \bigg)\\
&\leq \left( 4\alpha^2 +\left( 2C_r+\alpha \max_{z,y' \in \X, b \in A}\left| Q^*(z,b)-Q^*(y',b)\right|\right)^2 \right) \\
&\hspace{1cm}\cdot \left(1+\|\Delta_t\|_{\infty}^2\right) \leq  C \cdot \left(1+\|\Delta_t\|_{\infty}\right)^2.
\end{aligned}
\end{equation*}
This means the assumptions of Lemma~\ref{lem_convergence} are fulfilled, and we obtain that $\Delta_t(x,a) \rightarrow 0$ for $t\rightarrow \infty$  $\widehat{\PP}_{x_0,\ab}$-almost surely, which implies, by definition of $\Delta_t$, that $Q_t(x,a) \rightarrow Q^*(x,a)$ for $t \rightarrow \infty$  $\widehat{\PP}_{x_0,\ab}$-almost surely.
\qed

\textbf{Proof of Proposition~\ref{prop_bounds}}
The conditions of \cite[Assumption 2.1-2.4]{neufeld2023bounding} are satisfied w.r.t.\,$L_P$ and $L_r$ defined in \eqref{Lp} and \eqref{Lr} since here both the state and action space are finite. Hence the result follows from \cite[Theorem 3.1]{neufeld2023bounding}.
\qed
\end{document}